\documentclass{article} 
\usepackage{iclr2020_conference,times}

\usepackage{amsmath,amsfonts,bm}

\def\eqref#1{equation~\ref{#1}}

\def\1{\bm{1}}

\DeclareMathAlphabet{\mathsfit}{\encodingdefault}{\sfdefault}{m}{sl}
\SetMathAlphabet{\mathsfit}{bold}{\encodingdefault}{\sfdefault}{bx}{n}

\usepackage{hyperref}
\usepackage{url}

\usepackage{booktabs}       
\usepackage{amsfonts}       
\usepackage{nicefrac}       
\usepackage{microtype}      

\usepackage{graphicx}
\usepackage{subfigure}
\usepackage{color}
\usepackage{amsmath}
\usepackage{mathtools}

\usepackage{algorithm}
\usepackage{algorithmic}

\usepackage{amsthm}

\newcommand{\ep}{\mathbb{E}}

\makeatletter
\newtheorem*{rep@theorem}{\rep@title}
\newcommand{\newreptheorem}[2]{%
\newenvironment{rep#1}[1]{%
 \def\rep@title{#2 \ref{##1}}%
 \begin{rep@theorem}}%
 {\end{rep@theorem}}}
\makeatother

\newtheorem{theorem}{Theorem}
\newreptheorem{theorem}{Theorem}
\newtheorem{lemma}{Lemma}
\newreptheorem{lemma}{Lemma}
\newtheorem{proposition}{Proposition}
\newreptheorem{proposition}{Proposition}

\title{To Relieve Your Headache of Training an MRF, Take AdVIL}

\author{
Chongxuan Li\thanks{Dept. of Comp. Sci. \& Tech., BNRist Center, Institute for AI, THBI Lab, Tsinghua University, Beijing, 100084, China}~~,~Chao Du$^*$,~Kun Xu$^*$,~Max Welling\thanks{University of Amsterdam, and the Canadian Institute for Advanced Research (CIFAR).}~~,~Jun Zhu$^*$,~Bo Zhang$^*$ \\
\texttt{ \{chongxuanli1991,duchao0726,kunxu.thu\}@gmail.com,} \\
\texttt{ M.Welling@uva.nl,~\{dcszj,dcszb\}@mail.tsinghua.edu.cn}
}

\iclrfinalcopy 
\begin{document}

\maketitle

\begin{abstract}
We propose a black-box algorithm called {\it Adversarial Variational Inference and Learning} (AdVIL)  to perform inference and learning in a general Markov random field (MRF). AdVIL employs two variational distributions to approximately infer the latent variables and estimate the partition function of an MRF, respectively. The two variational distributions provide an estimate of the negative log-likelihood of the MRF as a minimax optimization problem, which is solved by stochastic gradient descent. AdVIL is proven convergent under certain conditions. On one hand, compared to the contrastive divergence, AdVIL requires minimal assumptions about the model structure and can deal with a broader family of MRFs. On the other hand, compared to existing black-box methods, AdVIL provides a tighter estimate of the log partition function and achieves much better empirical results. 
\end{abstract}

\section{Introduction}

Markov random fields (MRFs) find applications in a variety of machine learning areas~\citep{krahenbuhl2011efficient,salakhutdinov2010efficient,lafferty2001conditional}.
In particular, one famous example is conditional random fields~\citep{lafferty2001conditional}, a conditional version of MRFs that was developed to address the limitations (e.g., local dependency and label bias) of directed models for sequential data (e.g., hidden Markov models and other discriminative Markov models based on directed graphical models).
However, the inference and learning of general MRFs are challenging due to the presence of a global normalizing factor, i.e. partition function, especially when latent variables are present. 
Extensive efforts have been devoted to developing approximate methods. On one hand, sample-based methods~\citep{neal1993probabilistic} and variational approaches~\citep{jordan1999introduction,welling2005learning,salakhutdinov2010efficient} are proposed to infer the latent variables. On the other hand, extensive work~\citep{meng1996simulating,neal2001annealed,hinton2002training,tieleman2008training,wainwright2005new,wainwright2006log} has been done 
to estimate the partition function. Among these methods, contrastive divergence~\citep{hinton2002training} is proven effective in certain types of models.

Most of the existing methods highly depend on the model structure and require model-specific analysis in new applications, which makes it important to develop  black-box inference and learning methods. Previous work~\citep{ranganath2014black,schulman2015gradient} shows the ability to automatically infer the latent variables and obtain gradient estimate in directed models. However, there is no black-box learning method for undirected models except the recent work of NVIL~\citep{kuleshov2017neural}.

NVIL introduces a variational distribution and derives an upper bound of the partition function in a general MRF, in the same spirit as amortized inference~\citep{kingma2013auto,rezende2014stochastic,mnih2014neural} for directed models. NVIL has several advantages over existing methods, including the ability of black-box learning, tracking the partition function during training and getting approximate samples efficiently during testing. However, NVIL also comes with two disadvantages: (1) it leaves the inference problem of MRFs unsolved\footnote{NVIL~\citep{kuleshov2017neural} presents a hybrid model. The “inference” in the title refers to directed part but not for an MRF.} and only trains simple MRFs with tractable posteriors, and (2) the upper bound of the partition function can be underestimated~\citep{kuleshov2017neural}, resulting in sub-optimal solutions on high-dimensional data.

We propose {\it Adversarial Variational Inference and Learning} (AdVIL) to relieve some headache of learning an MRF model. AdVIL is a  black-box inference and learning method that partly solves the two problems of NVIL and retains the advantages of NVIL at the same time. First, AdVIL introduces a variational encoder to infer the latent variables, which provides an upper bound of the free energy. Second, AdVIL introduces a variational decoder for the MRF, which provides a lower bound of the log partition function. The two variational distributions provide an estimate of the negative log-likelihood of the MRF.
On one hand, the estimate is in an intuitive form of an approximate {\it contrastive free energy}, which is expressed in terms of the expected energy and the (conditional) entropy of the corresponding variational distribution.
On the other hand, similar to GAN~\citep{goodfellow2014generative}, the estimate is a minimax optimization problem, which is solved by stochastic gradient descent (SGD) in an alternating manner. Theoretically, our algorithm is convergent if the variational decoder approximates the model well. This motivates us to introduce an auxiliary variable to enhance the flexibility of the variational decoder, whose entropy is approximated by the third variational trick.

We evaluate AdVIL in various undirected generative models, including restricted Boltzmann machines (RBM)~\citep{ackley1985learning},  deep Boltzmann machines (DBM)~\citep{salakhutdinov2009deep}, and  Gaussian restricted Boltzmann machines (GRBM)~\citep{hinton2006reducing}, on several real datasets. We empirically demonstrate that (1) compared to the black-box NVIL~\citep{kuleshov2017neural} method, AdVIL provides a tighter estimate of the log partition function and achieves much better log-likelihood results; and (2) compared to contrastive divergence based methods~\citep{hinton2002training,welling2005learning}, AdVIL can deal with a broader family of MRFs without model-specific analysis and obtain better results when the model structure gets complex as in DBM.

\vspace{-.2cm}
\section{Background}
\vspace{-.1cm}

We consider a general case where the model consists of both visible variables $v$ and latent variables $h$. An MRF defines the joint distribution over $v$ and $h$ as
$ P(v, h) = \frac{e^{-  \mathcal{E}(v, h)}}{\mathcal{Z}},$
where $\mathcal{E}$ denotes the associated energy function that assigns a scalar value for a given configuration of $(v, h)$ and $\mathcal{Z}$ is the partition function such that $\mathcal{Z} = \int_{v, h} e^{-\mathcal{E}(v, h)}dv dh$.

Let $P_\mathcal{D}(v)$ denote the empirical distribution of the training data. Minimizing the negative log-likelihood (NLL) of an MRF is a commonly chosen learning criterion and it is given by:
\begin{align}
\mathcal{L}(\theta) \coloneqq  & -
\mathbb{E}_{P_\mathcal{D}(v)}\left[\log \int_h \frac{e^{-  \mathcal{E}(v, h)}}{\mathcal{Z}} dh\right], \label{eqn:nll}
\end{align}
where $\theta$ denotes the trainable parameters in $\mathcal{E}$ . Further, the gradient of $\theta$ is:
\begin{equation}
    \nabla_{\theta} \mathcal{L}(\theta)  = 
\mathbb{E}_{P_\mathcal{D}(v)}\left[ \nabla_{\theta}  \mathcal{F}(v) \right] - \mathbb{E}_{P(v)}  \left[ \nabla_{\theta}  \mathcal{F}(v)\right], \label{eqn:grad_nll} 
\end{equation}
where $\mathcal{F}(v) = - \log \int_h e^{-\mathcal{E}(v, h)} dh$ denotes the free energy and the gradient in Eqn.~(\ref{eqn:grad_nll}) is the difference of the free energy in two phases.
In the first {\it positive phase}, the expectation of the free energy under the data distribution is decreased. In the second {\it negative phase}, the expectation of the  free energy under the model distribution is increased.

Unfortunately, both the NLL in Eqn.~(\ref{eqn:nll}) and its gradient in Eqn.~(\ref{eqn:grad_nll}) are intractable in general for two reasons. First, the integral of the latent variables in 
Eqn.~(\ref{eqn:nll}) or equivalently the computation of the free energy in Eqn.~(\ref{eqn:grad_nll}) is intractable. Second, the computation of the partition function in Eqn.~(\ref{eqn:nll}) or equivalently the negative phase in Eqn.~(\ref{eqn:grad_nll}) is intractable.

\textbf{Variational inference.} Extensive work introduces deterministic approximations for the intractability of inference, including the mean-field approximation~\citep{welling2002new,salakhutdinov2009deep}, the Kikuchi and Bethe approximations~\citep{welling2005learning} and the recognition model approach~\citep{salakhutdinov2010efficient}. 
In this line of work, the intractability of the partition function is addressed using Monte Carlo based methods.

\textbf{Contrastive free energy.} Contrastive divergence (CD)~\citep{hinton2002training} addresses the intractability of the partition function by approximating the negative phase in Eqn.~(\ref{eqn:grad_nll}) as follows:
\begin{equation}
    \nabla_{\theta} \mathcal{L}(\theta)  = 
\mathbb{E}_{P_\mathcal{D}(v)}\left[ \nabla_{\theta}  \mathcal{F}(v) \right] - \mathbb{E}_{P_{CD}(v)} \left[ \nabla_{\theta}  \mathcal{F}(v)\right], \label{eqn:grad_cd}
\end{equation}
where $P_{CD}(v)$ denotes the empirical distribution obtained by starting from a data point and running several steps of Gibbs sampling according to the model distribution and the free energy $\mathcal{F}(v)$ is assumed to be tractable. Existing methods~\citep{welling2002new,welling2005learning} approximate $\mathcal{F}(v)$  using certain function $\mathcal{G}(v)$ and the gradient of $\theta$ is:
\begin{equation}
    \nabla_{\theta} \mathcal{L}(\theta)  \approx 
\mathbb{E}_{P_\mathcal{D}(v)}\left[ \nabla_{\theta}  \mathcal{G}(v) \right] - \mathbb{E}_{P_{CD}(v)} \left[ \nabla_{\theta}  \mathcal{G}(v)\right]. \label{eqn:grad_vcd}
\end{equation}
Although these generalized methods exist, it is nontrivial to extend CD-based methods to general MRFs because the Gibbs sampling procedure is highly dependent on the model structure.

\textbf{Black-box learning.}
The recent work of NVIL~\citep{kuleshov2017neural} addresses the intractability of the partition function in a black-box manner via a variational upper bound of the partition function:
\begin{equation}
    \mathbb{E}_{q(v)} \left[ \frac{ \tilde{P} (v)^2}{q(v)^2}\right ] \ge \mathcal{Z}^2,~\label{eqn:nvil_bound}
\end{equation}
where $\tilde{P} (v) = e^{-\mathcal{F}(v)}$ is the unnormalized marginal distribution on $v$ and $q(v)$ is a neural variational distribution. As a black-box learning method, NVIL potentially allows application to broader model families and improves the capabilities of probabilistic programming systems~\citep{carpenter2017stan}.
Though promising, NVIL leaves the intractability of inference in an MRF unsolved, and the bound in Eqn.~(\ref{eqn:nvil_bound}) is of high variance and is easily underestimated~\citep{kuleshov2017neural}. 

\begin{figure}[t]
\vspace{-.2cm}
    \centering
    \includegraphics[width=.9\columnwidth]{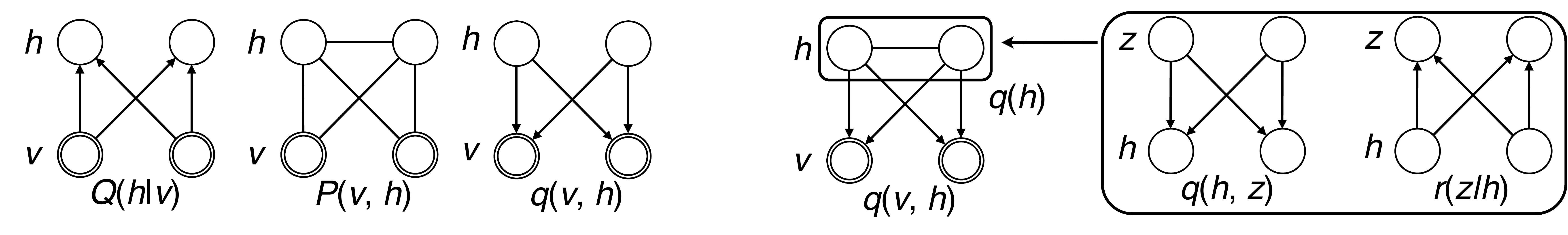}
\vspace{-.1cm}
    \caption{Illustration of the models involved in AdVIL. From left to right: variational encoder $Q(h|v)$, MRF $P(v, h)$, variational decoder $q(v, h)$ with a simple prior and $q(v, h)$ with an expressive prior. }
    \label{fig:AdVIL}
\vspace{-.2cm}
\end{figure}

\vspace{-.1cm}
\section{Method}
\vspace{-.1cm}

As stated above, the black-box inference and learning of MRFs are still largely open. In this paper, we make a step towards solving the problems by a new variational approach. For simplicity,  we focus on the resulting objective function in this section. See Appendix~\ref{app:obj} for detailed derivation.

\subsection{Adversarial variational inference and learning}

First, we rewrite the NLL of the MRF (See an illustration in Fig.~\ref{fig:AdVIL}) as follows:
\begin{equation}
\mathcal{L}(\theta) = - \mathbb{E}_{P_\mathcal{D}(v)}\left[- \mathcal{F}(v)\right] + \log \mathcal{Z}, \label{eqn:nll_2}
\end{equation}
where the negative free energy and the log partition function are in the form of a logarithm of an integral. Naturally, we can apply the variational trick~\citep{jordan1999introduction} twice and approximate the two terms individually. Due to the presence of the minus before the first term in Eqn.~(\ref{eqn:nll_2}), the two variational tricks bound the two parts of the NLL in the opposite directions, detailed as below.

Formally, on one hand, we introduce an approximate posterior for the latent variables $Q(h|v)$, which is parameterized as a {\it neural variational encoder} (See an illustration in Fig.~\ref{fig:AdVIL}), to address the intractability of inference as follows:
\begin{align}
\mathcal{L}(\theta) 
\le \mathbb{E}_
{P_\mathcal{D}(v)Q(h|v)} \left[\mathcal{E} (v, h) + \log Q(h | v) \right] + \log \mathcal{Z} \coloneqq \mathcal{L}_1(\theta, \phi), \label{eqn:AdVIL_step1}
\end{align}
where $\phi$ denotes the trainable parameters in $Q(h|v)$. 
The upper bound is derived via applying the Jensen inequality and the equality holds if and only if $Q(h|v) = P(h|v)$ for all $v$.
In the bound, the first term is the expected energy, which encourages $Q(h|v)$ to infer latent variables that have low values of the energy function $\mathcal{E}(v, h)$, or equivalently high probabilities of $P(v, h)$. The second term corresponds to the negative conditional entropy of $Q(h|v)$, which increases the uncertainty of $Q(h|v)$. In the paper, we denote the conditional entropy of $Q(h|v)$ as $\mathcal{H}(Q) \coloneqq - \mathbb{E}_
{P_\mathcal{D}(v)Q(h|v)} [\log Q(h | v)]$. 

On the other hand, we introduce an approximate sampler $q(v, h)$, which is parameterized by a {\it neural variational decoder} (See  Fig.~\ref{fig:AdVIL}), to address the  intractability of the partition function as follows:
\begin{align} \label{eqn:AdVIL_obj}  
\!\!  \mathcal{L}_1(\theta, \phi) 
\!\! \ge \!\! \mathbb{E}_{P_\mathcal{D}(v)Q(h|v)}\!\!\! \left[\underbrace{\overbrace{\mathcal{E}(v, h)}^{\text{energy term}} \!+\!  \overbrace{\log Q(h|v)}^{\text{entropy term}}}_{\textbf{Positive Phase}}\right]
\!\!\! - \!\!\mathbb{E}_{q(v, h)}\!\!\! \left[\underbrace{\overbrace{\mathcal{E}(v, h)}^{\text{energy term}} \! + \! \overbrace{\log q(v, h)}^{\text{entropy term}}}_{\textbf{Negative Phase}}\right]\!\!\coloneqq  \mathcal{L}_2(\theta,\phi,\psi),
\end{align}
where $\psi$ denotes the trainable parameters in $q(v, h)$.
The lower bound is derived via applying the Jensen inequality as well, and the equality holds if and only if $q(v, h) = P(v, h)$. It can be seen that the lower bound given by $q(v, h)$ consists of the entropy (denoted as $\mathcal{H}(q)$) and energy terms, which is similar to the upper bound in Eqn.~(\ref{eqn:AdVIL_step1}), and the overall objective is in the form of approximate {\it contrastive free energy}~\citep{hinton2002training, welling2005learning}.
Because the double variational trick bounds the NLL in opposite directions as above, we have a minimax optimization problem:
\begin{equation}
\min_{\theta} \min_{\phi} \max_{\psi} \mathcal{L}_2(\theta,\phi,\psi).~\label{eqn:lp}
\end{equation}
The minimax formulation has been investigated in GAN~\citep{goodfellow2014generative} and it is interpreted as an adversarial game between two networks. We name our framework {\it adversarial variational inference and learning} (AdVIL) following the well-established literature.

Note that $\mathcal{L}_2(\theta, \phi, \psi)$ is neither an upper bound, nor a lower bound of $\mathcal{L}(\theta)$ due to the double variational trick. However, we argue that solving the optimization problem in Eqn.~(\ref{eqn:lp}) is reasonable because (1) it is equivalent to optimizing $\mathcal{L}(\theta)$ under the nonparametric assumption, which is similar to GAN~\citep{goodfellow2014generative}; and (2) it converges to a stationary point of $\mathcal{L}_1(\theta, \phi)$, which is an upper bound of $\mathcal{L}(\theta)$, under a weaker assumption, as stated in the following theoretical analysis.


\subsection{Theoretical analysis of AdVIL} \label{sec:theo_analy}

In this section, we present our main theoretical results and the proofs can be found in Appendix~\ref{app:theory}.
Firstly, similarly to GAN~\citep{goodfellow2014generative}, we can prove that $\mathcal{L}_2$ is a tight estimate of $\mathcal{L}$ under the nonparametric assumption, which is summarized in Proposition~\ref{thm:nop} in Appendix~\ref{app:theorey_nop}. However, the nonparametric assumption does not tolerate any approximation error between $P(v, h)$ and $q(v, h)$ during training and no guarantee can be obtained in finite steps. To this end, we establish a convergence theorem based on a weaker assumption that allows non-zero approximation error before convergence. A key insight is that the angle between $\frac{\partial \mathcal{L}_2(\theta, \phi, \psi)}{\partial \theta}$ and $\frac{\partial \mathcal{L}_1(\theta, \phi)}{\partial \theta}$ is positive if $q(v, h)$ approximates $P(v, h)$ well, as stated in the following Lemma~\ref{thm:pos_grad}.
\begin{replemma}{thm:pos_grad}
For any $(\theta, \phi)$, there exists a symmetric positive definite matrix $H$ such that $\frac{\partial \mathcal{L}_2(\theta, \phi, \psi)}{\partial \theta}  = H \frac{\partial \mathcal{L}_1(\theta, \phi)}{\partial \theta}$  under the  assumption: $||\sum_{v, h}\delta(v, h) \frac{\partial \mathcal{E}(v, h) }
{\partial \theta}||_2  < || \frac{\partial \mathcal{L}_1(\theta, \phi)}{\partial \theta} ||_2$ if  $ || \frac{\partial \mathcal{L}_1(\theta, \phi)}{\partial \theta} ||_2 > 0 $ and $||\sum_{v, h}\delta(v, h) \frac{\partial \mathcal{E}(v, h) }
{\partial \theta}||_2 = 0$ if 
$|| \frac{\partial \mathcal{L}_1(\theta, \phi)}{\partial \theta} ||_2 = 0$,
where $\delta(v, h) = q(v, h) - P(v, h)$.
\end{replemma}

Based on Lemma~\ref{thm:pos_grad} and other commonly used assumptions in the analysis of stochastic optimization~\citep{bottou2018optimization}, AdVIL converges to a stationary point of $\mathcal{L}_1(\theta, \phi)$, as stated in Theorem~\ref{thm:convergence}.
\begin{reptheorem}{thm:convergence}
Solving the optimization problem in Eqn.~(\ref{eqn:lp}) using stochastic gradient descent, then $(\theta, \phi)$ converges to a stationary point of $\mathcal{L}_1(\theta, \phi)$
under the assumptions in the general stochastic optimization~\citep{bottou2018optimization} and that the condition of Lemma~\ref{thm:pos_grad} holds in each step.
\end{reptheorem}

Please see Appendix~\ref{app:theorey_convergence} for a detailed and formal version of Theorem~\ref{thm:convergence}. Compared to Proposition~\ref{thm:nop} and the analysis in GAN~\citep{goodfellow2014generative}, Theorem~\ref{thm:convergence} has a weaker statement that AdVIL converges to a stationary point of the negative evidence lower bound (i.e., $\mathcal{L}_1$)  instead of  $\mathcal{L}$. Nevertheless, we argue that converging to $\mathcal{L}_1$ is sufficiently good for
variational approaches in general. Besides,  Theorem~\ref{thm:convergence} states that AdVIL can at least decrease $\mathcal{L}_1$ in expectation if the assumption holds in finite steps. Indeed, we empirically justify Theorem~\ref{thm:convergence}, as detailed in Appendix~\ref{app:test_lemma}. Theorem~\ref{thm:convergence} also provides insights for the implementation of AdVIL. Indeed, its assumption motivates us to use a sufficiently powerful $q(v, h)$ with neural networks and auxiliary variables, and update $q(v, h)$ multiple times per update of $P(v, h)$, as detailed in Sec.~\ref{sec:spe_v} and  Sec.~\ref{sec:analysis} respectively.

\subsection{Specifying the variational distributions}
\label{sec:spe_v}

To efficiently get samples, both variational distributions are directed models.
We use a directed neural network that maps $v$ to $h$ as the variational encoder $Q(h|v)$~\citep{kingma2013auto}.

As for the variational decoder, we first factorize it as the product of a prior over $h$ and a conditional distribution, namely $q(v, h) = q(v|h) q(h)$. It is nontrivial to specify the prior $q(h)$ because the marginal distribution of $h$ in the MRF, i.e. $P(h)$, can be correlated across different dimensions. Consequently, a simple $q(h)$ is not flexible enough to track $P(h)$ and can violate the condition of Lemma~\ref{thm:pos_grad}. To this end, we introduce an auxiliary variable $z$, which can be discrete or continuous, on top of $h$ and define $q(v, h) = \int_z q(z) q(h | z) q(v | h) dz$.\footnote{An alternative way is to use an autoregressive model as $q(h)$. See results and analysis in Appendix~\ref{app:nade}.} (See an illustration in Fig.~\ref{fig:AdVIL}.) However, the entropy term of $q(v, h)$ is intractable because we need to integrate out the auxiliary variable $z$. Therefore, we introduce the third variational distribution $r(z|h)$ to approximate the entropy of $q(v, h)$.
As in Eqn.~(\ref{eqn:AdVIL_step1}),  applying the standard variational trick gives an upper bound:
\begin{align}
 - \mathbb{E}_{q(v, h)} \log q(v, h) 
\le - \mathbb{E}_{ q(v, h)}  \log q(v| h) - \mathbb{E}_{  q(h) r(z|h)}  \log \left [\frac{q(h, z)}{r(z|h)} \right ],\label{eqn:entropy_lower_bound} 
\end{align}
which is unsatisfactory because the estimate is minimized w.r.t $r(z|h)$ while maximized w.r.t $q(v, h)$. Instead, after some transformations (See details in Appendix~\ref{app:obj}) we get a lower bound as follows:
\begin{align}
 - \mathbb{E}_{q(v, h)} \log q(v, h) 
\ge - \mathbb{E}_{ q(v, h)}  \log q(v| h) - \mathbb{E}_{  q(h, z)}  \log \left [\frac{q(h, z)}{r(z|h)} \right ].\label{eqn:entropy_bound}
\end{align}
The equality holds if and only if $r(z|h) = q(z|h)$ for all $h$. 
The difference between the two bounds is subtle: the last expectation in 
Eqn.~(\ref{eqn:entropy_lower_bound})
is over $q(h) r(z|h)$ but that in 
Eqn.~(\ref{eqn:entropy_bound}) is over $q(h, z)$.
Here, a lower bound is preferable because the estimate is maximized with respect to both $r(z|h)$ and $q(v, h)$ and we can train them simultaneously. For simplicity, we absorb the trainable parameters of $r(z|h)$ into $\psi$.
Note that after introducing $z$ and $r(z|h)$, we can still obtain a convergence theorem of AdVIL under the conditions that $r(z|h)$ approximates $q(z|h)$ well and $q(v, h) = \int q(v, h, z) dz$ is sufficiently close to $P(v, h)$ in every step, together the assumptions in general stochastic optimization.

Following GAN~\citep{goodfellow2014generative}, we optimize $\theta$, $\phi$ and $\psi$ jointly using stochastic gradient descent (SGD) in an alternating manner. The partial derivatives of $\phi$ and $\psi$ are estimated via the reparameterization trick~\citep{kingma2013auto} for the continuous variables and the Gumbel-Softmax trick~\citep{jang2016categorical,maddison2016concrete} for the discrete variables. See Algorithm~\ref{algo:AdVIL} in Appendix~\ref{app:algo} for the whole training procedure. Note that $\psi$ is updated $K_1 > 1$ times per update of $\theta$.

\vspace{-.2cm}
\section{Related work}
\label{sec:related_work}
\vspace{-.2cm}

Existing traditional methods~\citep{neal2001annealed,hinton2002training,winn2005variational,wainwright2006log,rother2007optimizing} can be used to estimate the log partition function but are nontrivial to be extended to learn general MRFs. Some methods~\citep{winn2005variational,neal2001annealed} require an expensive inference procedure for each update of the model and others~\citep{hinton2002training,rother2007optimizing} cannot be directly applied to general cases (e.g., DBM). 
Among these methods, contrastive divergence (CD)~\citep{hinton2002training} is proven effective in certain types of models and it is closely related to AdVIL. Indeed, the partial derivative of $\theta$ in AdVIL is:
\begin{equation}
    \frac{\partial  \mathcal{L}_2(\theta,\phi,\psi)}{ \partial \theta}  = \mathbb{E}_{P_\mathcal{D}(v)Q(h|v)}\left[ \frac{\partial}{ \partial \theta}  \mathcal{E}(v, h)\right]  - \mathbb{E}_{q(v, h)}\left[  \frac{\partial}{ \partial \theta}  \mathcal{E}(v, h)\right], \label{eqn:AdVIL_grad_e}
\end{equation}
which also involves a positive phase and a negative phase naturally and is quite similar to Eqn.~(\ref{eqn:grad_cd}).
However, notably, the two phases average over the $(v, h)$ pairs and only require the knowledge of the energy function without any further assumption of the model in AdVIL. Therefore, AdVIL is more suitable to general MRFs than CD (See empirical evidence in Sec.~\ref{sec:exp_dbm}).

In the context of black-box learning in MRFs, AdVIL competes directly with NVIL~\citep{kuleshov2017neural}.
It seems that the upper bound in Eqn.~(\ref{eqn:nvil_bound}) is suitable for optimization because $P$ and $q$ share the same training direction. However, the bound holds only if the support of $\tilde P$ is a subset of the support of $q$. 
Further, the Monte Carlo estimate of the upper bound is of high variance. Therefore,
the bound of NVIL can be easily underestimated, which results in sub-optimal solutions~\citep{kuleshov2017neural}.
In contrast, though AdVIL arrives at a minimax optimization problem, the estimate of Eqn.~(\ref{eqn:AdVIL_obj}) is tighter and of lower variance. 
We empirically verify this argument (See Fig.~\ref{fig:nvil_under_estimate}) and systematically compare the two methods (See Tab.~\ref{table:rbm_uci}) in Sec.5.4.

Apart from the work on approximate inference and learning in MRFs as mentioned above, AdVIL is also related to some directed  models.
\citet{kim2016deep} jointly trains a deep energy model~\citep{ngiam2011learning} and a directed  generative model by minimizing the KL-divergence between them. Similar ideas have been highlighted in~\citep{finn2016connection,zhai2016generative,dai2017calibrating,liu2017learning}.  In comparison, firstly, AdVIL obtains the objective function in a unified perspective on the black-box inference and learning in general MRFs. Note that dealing with latent variables in MRFs is nontrivial~\citep{kim2016deep} and therefore existing work focuses on fully observable models. Secondly, AdVIL uses a sophisticated decoder with auxiliary variables to handle the latent variables and derives a principled variational approximation of the entropy term instead of the heuristics~\citep{kim2016deep,zhai2016generative}. 
Lastly, the convergence of AdVIL is formally characterized by Theorem~\ref{thm:convergence} while the effect of the approximation error in inference is not well understood in existing methods. 
Adversarially learned inference (ALI)~\citep{donahue2016adversarial,dumoulin2016adversarially} is also formulated as a minimax optimization problem but focuses on directed models.

\vspace{-.2cm}
\section{Experiments}
\vspace{-.2cm}

In this section, we evaluate AdVIL in restricted Boltzmann machines (RBM)~\citep{ackley1985learning}, deep Boltzmann machines (DBM)~\citep{salakhutdinov2009deep} and  Gaussian restricted Boltzmann machines (GRBM)~\citep{hinton2006reducing} 
on the Digits dataset, the UCI binary databases~\citep{Dua:2017} and the Frey faces datasets (See detailed settings in Appendix~\ref{app:data} and the source code\footnote{See the source code in https://anonymous.4open.science/r/8c779fbc-6394-40c7-8273-e52504814703/.}).
We compare AdVIL with strong baseline methods systematically and show the promise of AdVIL to learn a broad family of models effectively as a
black-box method.

\begin{figure}[t]
\begin{center}
\subfigure[upper bound of $\mathcal{F}(v)$]{\includegraphics[width=0.24\columnwidth]{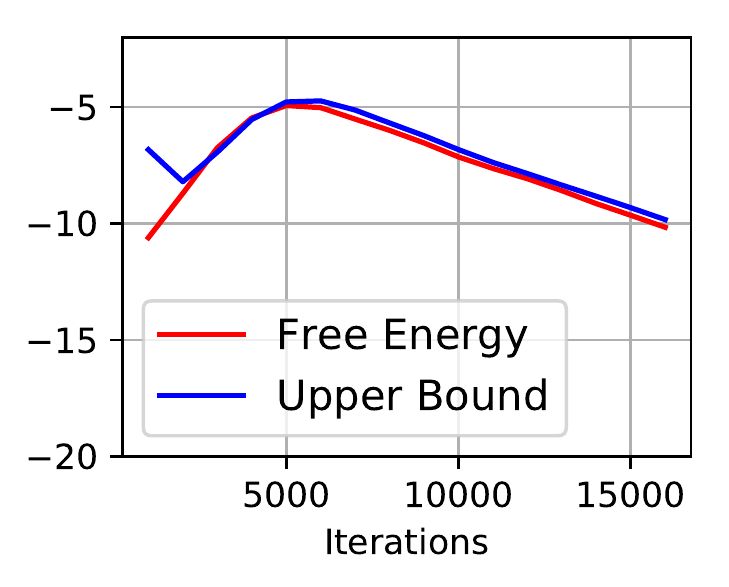}}
\subfigure[lower bound of $\mathcal{H}(q)$ ]{\includegraphics[width=0.24\columnwidth]{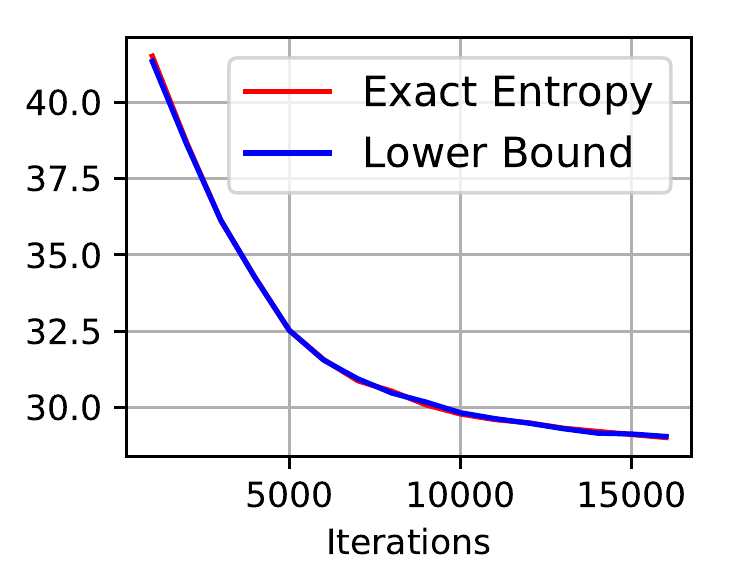}}
\subfigure[lower bound of $\log \mathcal{Z}$ ]{\includegraphics[width=0.24\columnwidth]{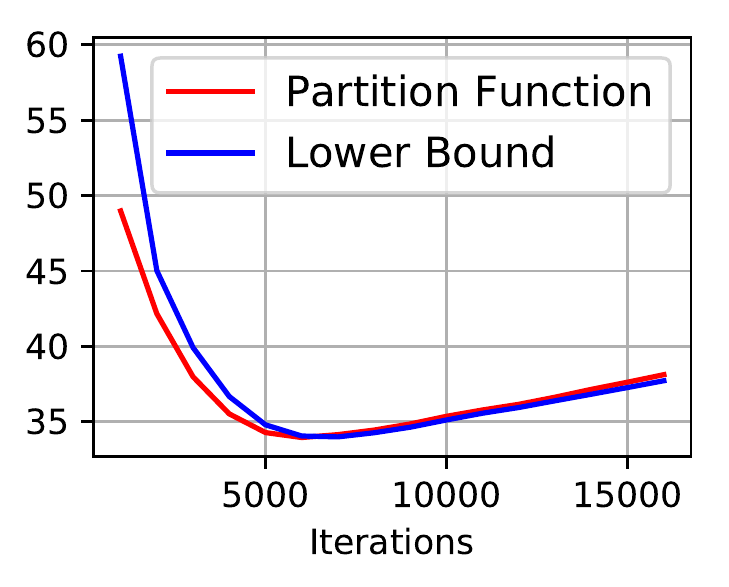}}
\subfigure[RBM loss and NLL]{\includegraphics[width=0.24\columnwidth]{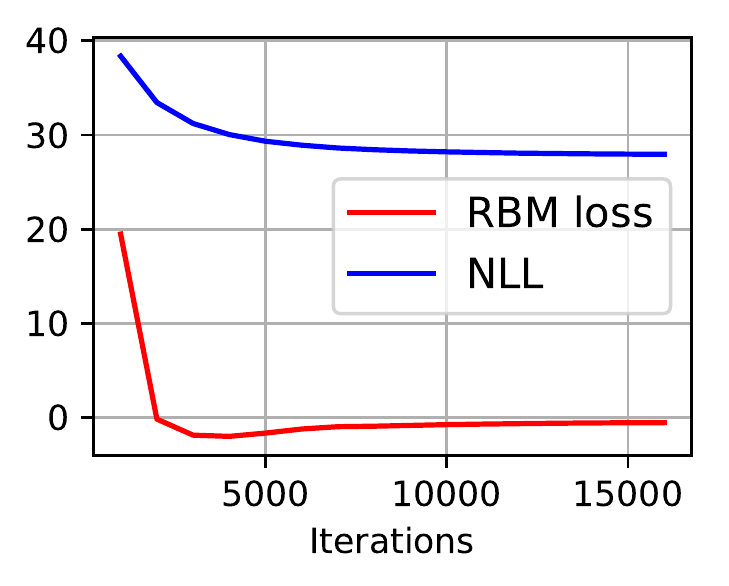}}
\vspace{-.2cm}
\caption{Curves of AdVIL on Digits. (a-c) compare the values of the variational approximations and the corresponding ground truths. All bounds are rather tight after 5,000 iterations. (d) shows that the RBM loss (i.e., the loss of  $\theta$ as in Eqn.~(\ref{eqn:AdVIL_obj})) tends to zero and the model converges gradually. 
}
\vspace{-.3cm}
\label{fig:analysis}
\end{center}
\end{figure}

\vspace{-.1cm}
\subsection{Empirical analysis of AdVIL}\label{sec:analysis}
\vspace{-.1cm}

We present a detailed analysis of AdVIL in RBM, whose  energy function is defined as $
\mathcal{E}(v, h) = - b^\top  v - v^\top  W h - c^\top h.$ The conditional distributions of an RBM are tractable, but we still treat $P(h|v)$ as unknown and train AdVIL in a fully black-box manner. 
The analysis is performed on the Digits dataset and we augment the data of five times by shifting the digits following the protocol in~\citep{kuleshov2017neural}. The dimensions of $v$, $h$ and $z$ are $64$, $15$ and $10$, respectively. Therefore, the log partition function of the RBM and the entropy of the decoder can be computed by brute force.

Firstly, we empirically validate AdVIL in Fig.~\ref{fig:analysis}. 
Specifically, Panel (a) shows that the variational encoder $Q(h|v)$ provides a tight upper bound of the free energy after 2,000 iterations. Panel (b) demonstrates that the variational distribution $r(z|h)$ estimate the entropy of $q(v, h)$ accurately. Panel (c) shows that $q(v, h)$ can successfully track the log partition function after 5,000 iterations. Panel (d) presents that the RBM loss balances well between the negative phase and positive phase, and the model converges gradually. See Appendix~\ref{app:test_lemma} for an empirical test of the condition in Lemma~\ref{thm:pos_grad}.

Secondly, we empirically show that both $P$ and $q$ can generate data samples in Appendix~\ref{app:rbm_samples}.

Lastly, we analyze the sensitivity of $K_1$. Theoretically, enlarging $K_1$ will make $q(v, h)$ and $P(v, h)$ to be close and then help the convergence according to Theorem~\ref{thm:convergence}. As shown in Fig.~\ref{fig:nvil_under_estimate} (a), a larger $K_1$ at least won't hurt the convergence, which agrees with Theorem~\ref{thm:convergence}. Though 
$K_1 = 15$ is sufficient on the Digits dataset, we use $K_1 = 100$ as a default setting for AdVIL on larger datasets.

\begin{figure} 
\vspace{.3cm}
\centering
\subfigure[Sensitivity of $K_1$]{\includegraphics[width=0.24\columnwidth]{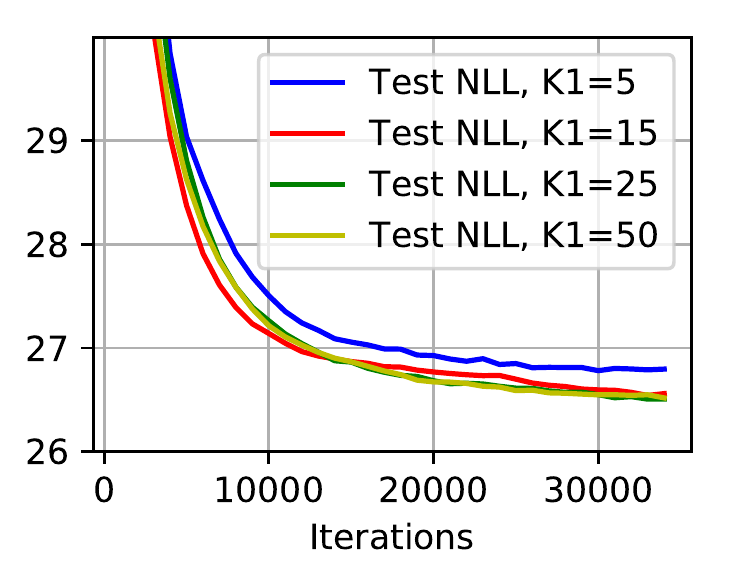}}
\subfigure[NVIL]{\includegraphics[width=0.24\columnwidth]{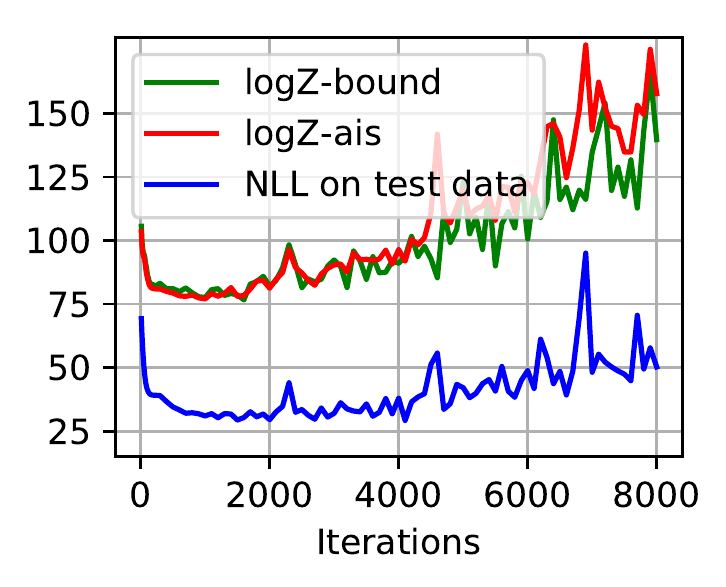}}
\subfigure[AdVIL]{\includegraphics[width=0.24\columnwidth]{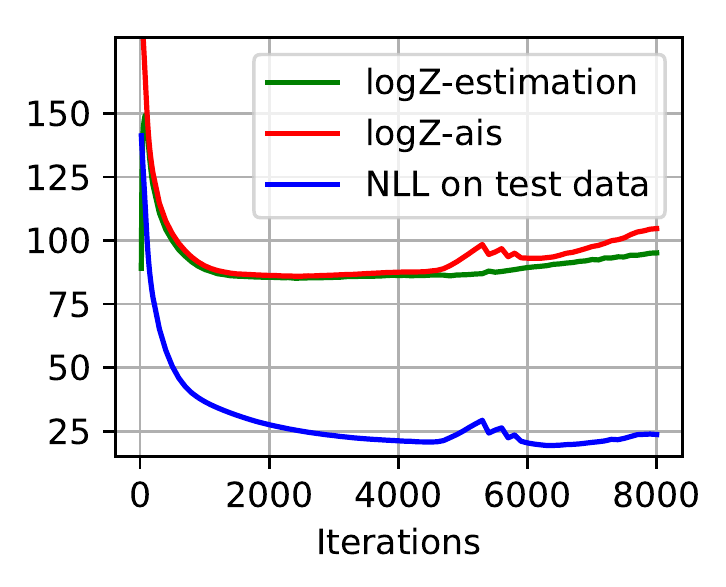}}
\subfigure[PCD-1]{\includegraphics[width=0.24\columnwidth]{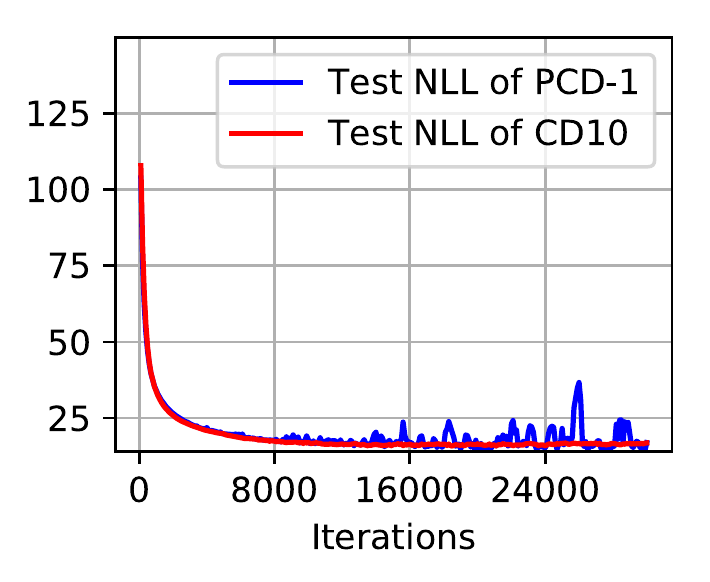}}
\vspace{-.2cm}
\caption{(a) Sensitivity analysis of $K_1$ on the Digits dataset. (b-d) Learning curves of  NVIL,  AdVIL and CD on the Mushrooms dataset. Compared to NVIL, AdVIL provides a tighter and lower variance estimate of $\log \mathcal{Z}$ and achieves better performance.  Compared to PCD-1 and CD-10, AdVIL can track the log partition function and achieve comparable results though trained in a black-box manner. }
\vspace{-.3cm}
\label{fig:nvil_under_estimate}
\end{figure}

\begin{table*}[t]
\vspace{-.1cm}
\setlength{\tabcolsep}{4pt}
  \centering
  	\caption{Anneal importance sampling (AIS) results in RBM. The results are recorded on the test set according to the best validation performance and averaged over three runs. 
  	AdVIL outperforms NVIL consistently and significantly. See the standard deviations in Appendix~\ref{app:std}.}
\label{table:rbm_uci}
\vskip 0.05in
 \resizebox{\textwidth}{8mm}{
  \begin{tabular}{cccccccccc}
      \toprule
    	Method & Digits & Adult & Connect4 & DNA & Mushrooms &  NIPS-0-12 & Ocr-letters &  RCV1  \\
    \midrule
	NVIL-mean & $-$27.36 & $-$20.05 & $-$24.71 & $-$97.71 & $-$29.28 &$-$290.01& $-$47.56 & $-$50.47 \\
	AdVIL-mean &  $\mathbf{-26.34}$ & $\mathbf{-19.29}$ & $\mathbf{-21.95}$ & $\mathbf{-97.59}$ & $\mathbf{-19.59}$ & $\mathbf{-276.42}$ & $\mathbf{-45.64}$ & $\mathbf{-50.22}$\\
	
\bottomrule
  \end{tabular}
  }
  \vspace{-.4cm}
\end{table*}


\vspace{-.1cm}
\subsection{RBM Results}
\label{sec:exp_rbm}
\vspace{-.1cm}

To the best of our knowledge, NVIL~\citep{kuleshov2017neural} is the only existing black-box learning method for MRFs and hence it is the most direct competitor of AdVIL. In this section, we provide a systematic comparison and analysis of these two methods in terms of the log-likelihood results on the UCI databases~\citep{Dua:2017}.

For a fair comparison, we use the widely-adopted anneal importance sampling (AIS)~\citep{salakhutdinov2008quantitative} metric for quantitative evaluation. Besides, we carefully perform grid search over the default settings of NVIL~\citep{kuleshov2017neural} and our settings based on their code, and choose the best configuration including $K_1 = 100$ (See details in Appendix~\ref{app:data}). We directly compare with the best version of NVIL in Tab.~\ref{table:rbm_uci}. It can be seen that AdVIL consistently outperforms NVIL on all datasets, which demonstrate the effectiveness of AdVIL. Besides, the time complexity of AdVIL is comparable to that of NVIL with the same hyperparameters.

We compare the learning curves of NVIL and AdVIL on the Mushroom dataset. 
As shown in Fig.~\ref{fig:nvil_under_estimate} (b), the upper bound of NVIL is underestimated after 4,000 iterations and then the model can get worse or even diverge. 
In contrast, as shown in Fig.~\ref{fig:nvil_under_estimate} (c), the lower bound of AdVIL is consistently valid. Besides, the estimate of NVIL is looser and of higher variance than that of AdVIL. 
The results agree with our analysis in Sec.~\ref{sec:related_work} and explain why AdVIL significantly outperforms NVIL. Further, as shown in Fig.~\ref{fig:nvil_under_estimate} (d), AdVIL is comparable to CD-10 and persistent contrastive divergence (PCD)~\citep{tieleman2008training}, which leverage the tractability of the conditional distributions in an RBM.

\begin{table*}[t]
\vspace{-.1cm}
\setlength{\tabcolsep}{4pt}
	\caption{AIS results in DBM. The results are recorded according to the best validation performance and averaged by three runs. AdVIL achieves higher averaged AIS results on five out of eight datasets and has a better overall performance than VCD. See the standard deviations in Appendix~\ref{app:std}.}
\label{table:dbm_uci}
\vskip 0.05in
  \centering
 \resizebox{\textwidth}{8mm}{
  \begin{tabular}{cccccccccc}
      \toprule
    	Method & Digits & Adult & Connect4 & DNA & Mushrooms &  NIPS-0-12 & Ocr-letters &  RCV1 \\
    \midrule
    VCD-mean   & $-28.49$ & $-22.26$ & $-26.79$ & $\mathbf{-97.59}$ & $-23.15$ & $-356.26$  & $\mathbf{-45.77}$ & $\mathbf{-50.83}$ \\
    AdVIL-mean & $\mathbf{-27.89}$  & $\mathbf{-20.29}$ & $\mathbf{-26.34}$ & $-99.40$ & $\mathbf{-21.21}$ & $\mathbf{-287.15}$  & $-48.38$ & $-51.02$ \\
    
\bottomrule
  \end{tabular}
  }
\end{table*}

\vspace{-.1cm}
\subsection{DBM Results}\label{sec:exp_dbm}
\vspace{-.1cm}

We would like to demonstrate that AdVIL has the ability to deal with highly intractable models such as a DBM conveniently and effectively, compared to standard CD-based methods~\citep{hinton2002training,welling2002new, welling2005learning} and NVIL~\citep{kuleshov2017neural}.

DBM~\citep{salakhutdinov2009deep} is a powerful family of deep models that stack multiple RBMs together. The energy function of a two-layer DBM is defined as $\mathcal{E}(v, h_1, h_2) = - b^\top  v - v^\top  W_1 h_1  - c_1^\top h_1-  h_1^\top  W_2 h_2 -  c_2^\top h_2.$
Learning a DBM is challenging because $P(h_1, h_2| v)$ is not tractable and CD~\citep{hinton2002training} is not applicable. Inspired by~\citep{welling2002new, welling2005learning}, we construct a variational CD (VCD) baseline by employing the same variational encoder $Q(h_1, h_2 | v)$ as in AdVIL. The free energy is approximated by the same upper bound as in Eqn.~(\ref{eqn:AdVIL_step1}), which is minimized with respect to the parameters in $Q(h_1, h_2 | v)$. The gradient of the parameters in the DBM is given by Eqn.~(\ref{eqn:grad_vcd}), where the Gibbs sampling procedure is approximated by $h_1 \sim Q(h_1 | v)$ and $v \sim P(v|h_1)$. Note that AdVIL can be directly applied to this case. As for the time complexity, the training speed of AdVIL is around ten times slower than that of VCD in our implementation. However, the approximate inference and sampling procedure of AdVIL is very efficient thanks to the directed variational distributions.

The log-likelihood results on the UCI databases are shown in Tab.~\ref{table:dbm_uci}. It can be seen that AdVIL  has a better overall performance even trained in a black-box manner, which shows the promise of AdVIL. See Appendix~\ref{app:ana_dbm} for learning curves and a detailed analysis of the results.

We also extend NVIL by using the same $Q(h_1, h_2 | v)$ and $q(v, h_1, h_2)$ as AdVIL. However, NVIL diverges after 300 iterations and gets bad AIS results (e.g., less than $-40$ on Digits) in our implementation. A potential reason is that the upper bound given by $q$ in NVIL can be underestimated if $q$ is high-dimensional, as analyzed in Sec.~\ref{sec:related_work} and Fig.~\ref{fig:nvil_under_estimate}. Note that $q(v, h_1, h_2)$ in DBM involves latent variables and has a higher dimension (e.g. 164 on the Digits dataset) than $q(v)$ in RBM (e.g. 64 on  the Digits dataset). The results again demonstrate the advantages of AdVIL over NVIL.

\vspace{-.1cm}
\subsection{GRBM results}
\vspace{-.1cm}

\begin{figure}[t]
\begin{center}
\vspace{-.2cm}
\subfigure[Data]{\includegraphics[width=0.24\columnwidth]{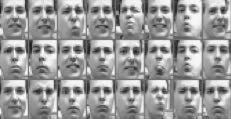}}
\subfigure[Filters of the GRBM]{\includegraphics[width=0.24\columnwidth]{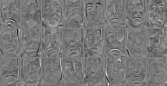}}
\subfigure[Samples from $q$]{\includegraphics[width=0.24\columnwidth]{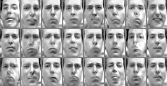}}
\subfigure[Samples from $P$]{\includegraphics[width=0.24\columnwidth]{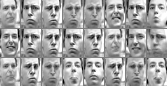}}
\vspace{-.2cm}
\caption{Filters and samples of a GRBM learned by AdVIL on the Frey faces dataset. (a) presents the training data. (b) presents the first 40 filters of the GRBM. (c) and (d) show random samples from the variational decoder and the GRBM, respectively. We present the mean of $v$ for better visualization.}
\label{fig:grbm}
\vspace{-.3cm}
\end{center}
\end{figure}

We now show the ability of AdVIL to learn a GRBM on the continuous Frey faces dataset. 
The energy function of a GRBM is $\mathcal{E}(v, h) = \frac{1}{2\sigma^2} ||v - b||^2 - c^\top  h - \frac{1}{\sigma}v^\top  W h,$ where $\sigma$ is the standard deviation of the Gaussian likelihood and is set as $1$ manually. We standardize the data by subtracting the mean and dividing by the standard deviation. The dimensions of $h$ and $z$ are 200 and 50, respectively.

Though a GRBM is more sensitive to the hyperparameters and hence harder to train than an RBM~\citep{cho2011improved,cho2013gaussian}, AdVIL can successfully capture the underlying data distribution using the default hyperparameters (See Appendix~\ref{app:data}). As shown in Fig.~\ref{fig:grbm}, the samples from both the GRBM (via Gibbs sampling after 100,000 burn-in steps) and the decoder are  meaningful faces. Besides, the filters of the GRBM outline diverse prototypes of faces, which accords with our expectation.

In summary, the results of the three models together demonstrate that AdVIL can learn a broad family of models conveniently and effectively in a fully black-box manner.

\vspace{-.2cm}
\section{Conclusion and discussion}
\vspace{-.2cm}

A novel black-box learning and inference method for undirected graphical models, called adversarial variational inference and learning (AdVIL), is proposed.
The key to AdVIL is a double variational trick that approximates the negative free energy and the log partition function separately. A formal convergence theorem, which provides insights for implementation, is established for AdVIL.
Empirical results show that AdVIL can deal with a broad family of MRFs in a fully black-box manner and outperforms both the standard contrastive divergence method and the black-box NVIL algorithm.

Though AdVIL shows promising results, we emphasize that the black-box learning and inference of the MRFs are far from completely solved, especially on high-dimensional data. The two intractability problems of MRFs are distinct since the posterior of the latent variables is {\it local} in terms of $v$ but the partition function is {\it global} by integrating out $v$. The additional integral makes estimating the partition function much more challenging. In AdVIL, simply increasing the number of updates of the decoder to obtain a tighter estimate of the partition function on high-dimensional data can be expensive.  A potential future work to avoid the problem is adopting recent advances on non-convex optimization~\citep{dauphin2014identifying,reddi2016stochastic,wang2017stochastic} to accelerate the inner loop optimization. We conjecture that AdVIL is comparable to CD in RBM and superior to VCD in DBM on larger datasets if AdVIL can be trained to nearly converge based on our current results. 

\section*{Acknowledgements}

This work was supported by the National Key Research and Development Program of China (No. 2017YFA0700904), NSFC Projects (Nos. 61620106010, U19B2034, U1811461), Beijing NSF Project (No. L172037), Beijing Academy of Artificial Intelligence (BAAI), Tsinghua-Huawei Joint Research Program, a grant from Tsinghua Institute for Guo Qiang, Tiangong Institute for Intelligent Computing, the JP Morgan Faculty Research Program and the NVIDIA NVAIL Program with GPU/DGX Acceleration. C. Li was supported by the Chinese postdoctoral innovative talent support program and Shuimu Tsinghua Scholar.

\bibliographystyle{iclr2020_conference}
\bibliography{example}

\begin{thebibliography}{49}
\providecommand{\natexlab}[1]{#1}
\providecommand{\url}[1]{\texttt{#1}}
\expandafter\ifx\csname urlstyle\endcsname\relax
  \providecommand{\doi}[1]{doi: #1}\else
  \providecommand{\doi}{doi: \begingroup \urlstyle{rm}\Url}\fi

\bibitem[Abadi et~al.(2016)Abadi, Barham, Chen, Chen, Davis, Dean, Devin,
  Ghemawat, Irving, Isard, et~al.]{abadi2016tensorflow}
Mart{\'\i}n Abadi, Paul Barham, Jianmin Chen, Zhifeng Chen, Andy Davis, Jeffrey
  Dean, Matthieu Devin, Sanjay Ghemawat, Geoffrey Irving, Michael Isard, et~al.
\newblock Tensor{F}low: A system for large-scale machine learning.
\newblock 2016.

\bibitem[Ackley et~al.(1985)Ackley, Hinton, and Sejnowski]{ackley1985learning}
David~H Ackley, Geoffrey~E Hinton, and Terrence~J Sejnowski.
\newblock A learning algorithm for boltzmann machines.
\newblock \emph{Cognitive science}, 9\penalty0 (1):\penalty0 147--169, 1985.

\bibitem[Bottou et~al.(2018)Bottou, Curtis, and
  Nocedal]{bottou2018optimization}
L{\'e}on Bottou, Frank~E Curtis, and Jorge Nocedal.
\newblock Optimization methods for large-scale machine learning.
\newblock \emph{Siam Review}, 60\penalty0 (2):\penalty0 223--311, 2018.

\bibitem[Boyd \& Vandenberghe(2004)Boyd and Vandenberghe]{boyd2004convex}
Stephen Boyd and Lieven Vandenberghe.
\newblock \emph{Convex optimization}.
\newblock Cambridge university press, 2004.

\bibitem[Carpenter et~al.(2017)Carpenter, Gelman, Hoffman, Lee, Goodrich,
  Betancourt, Brubaker, Guo, Li, and Riddell]{carpenter2017stan}
Bob Carpenter, Andrew Gelman, Matthew~D Hoffman, Daniel Lee, Ben Goodrich,
  Michael Betancourt, Marcus Brubaker, Jiqiang Guo, Peter Li, and Allen
  Riddell.
\newblock Stan: A probabilistic programming language.
\newblock \emph{Journal of statistical software}, 76\penalty0 (1), 2017.

\bibitem[Cho et~al.(2013)Cho, Raiko, and Ilin]{cho2013gaussian}
Kyung~Hyun Cho, Tapani Raiko, and Alexander Ilin.
\newblock Gaussian-bernoulli deep boltzmann machine.
\newblock In \emph{Neural Networks (IJCNN), The 2013 International Joint
  Conference on}, pp.\  1--7. IEEE, 2013.

\bibitem[Cho et~al.(2011)Cho, Ilin, and Raiko]{cho2011improved}
KyungHyun Cho, Alexander Ilin, and Tapani Raiko.
\newblock Improved learning of gaussian-bernoulli restricted boltzmann
  machines.
\newblock In \emph{International conference on artificial neural networks},
  pp.\  10--17. Springer, 2011.

\bibitem[Dai et~al.(2017)Dai, Almahairi, Bachman, Hovy, and
  Courville]{dai2017calibrating}
Zihang Dai, Amjad Almahairi, Philip Bachman, Eduard Hovy, and Aaron Courville.
\newblock Calibrating energy-based generative adversarial networks.
\newblock \emph{arXiv preprint arXiv:1702.01691}, 2017.

\bibitem[Dauphin et~al.(2014)Dauphin, Pascanu, Gulcehre, Cho, Ganguli, and
  Bengio]{dauphin2014identifying}
Yann~N Dauphin, Razvan Pascanu, Caglar Gulcehre, Kyunghyun Cho, Surya Ganguli,
  and Yoshua Bengio.
\newblock Identifying and attacking the saddle point problem in
  high-dimensional non-convex optimization.
\newblock In \emph{Advances in neural information processing systems}, pp.\
  2933--2941, 2014.

\bibitem[Dheeru \& Karra(2017)Dheeru and Karra]{Dua:2017}
Dua Dheeru and Taniskidou~E Karra.
\newblock {UCI} machine learning repository, 2017.
\newblock URL \url{http://archive.ics.uci.edu/ml}.

\bibitem[Donahue et~al.(2016)Donahue, Kr{\"a}henb{\"u}hl, and
  Darrell]{donahue2016adversarial}
Jeff Donahue, Philipp Kr{\"a}henb{\"u}hl, and Trevor Darrell.
\newblock Adversarial feature learning.
\newblock \emph{arXiv preprint arXiv:1605.09782}, 2016.

\bibitem[Dumoulin et~al.(2016)Dumoulin, Belghazi, Poole, Mastropietro, Lamb,
  Arjovsky, and Courville]{dumoulin2016adversarially}
Vincent Dumoulin, Ishmael Belghazi, Ben Poole, Olivier Mastropietro, Alex Lamb,
  Martin Arjovsky, and Aaron Courville.
\newblock Adversarially learned inference.
\newblock \emph{arXiv preprint arXiv:1606.00704}, 2016.

\bibitem[Finn et~al.(2016)Finn, Christiano, Abbeel, and
  Levine]{finn2016connection}
Chelsea Finn, Paul Christiano, Pieter Abbeel, and Sergey Levine.
\newblock A connection between generative adversarial networks, inverse
  reinforcement learning, and energy-based models.
\newblock \emph{arXiv preprint arXiv:1611.03852}, 2016.

\bibitem[Goodfellow et~al.(2014)Goodfellow, Pouget-Abadie, Mirza, Xu,
  Warde-Farley, Ozair, Courville, and Bengio]{goodfellow2014generative}
Ian Goodfellow, Jean Pouget-Abadie, Mehdi Mirza, Bing Xu, David Warde-Farley,
  Sherjil Ozair, Aaron Courville, and Yoshua Bengio.
\newblock Generative adversarial nets.
\newblock In \emph{Advances in neural information processing systems}, pp.\
  2672--2680, 2014.

\bibitem[Heusel et~al.(2017)Heusel, Ramsauer, Unterthiner, Nessler, and
  Hochreiter]{heusel2017gans}
Martin Heusel, Hubert Ramsauer, Thomas Unterthiner, Bernhard Nessler, and Sepp
  Hochreiter.
\newblock Gans trained by a two time-scale update rule converge to a local nash
  equilibrium.
\newblock In \emph{Advances in Neural Information Processing Systems}, pp.\
  6626--6637, 2017.

\bibitem[Hinton(2002)]{hinton2002training}
Geoffrey~E Hinton.
\newblock Training products of experts by minimizing contrastive divergence.
\newblock \emph{Neural computation}, 14\penalty0 (8):\penalty0 1771--1800,
  2002.

\bibitem[Hinton \& Salakhutdinov(2006)Hinton and
  Salakhutdinov]{hinton2006reducing}
Geoffrey~E Hinton and Ruslan~R Salakhutdinov.
\newblock Reducing the dimensionality of data with neural networks.
\newblock \emph{science}, 313\penalty0 (5786):\penalty0 504--507, 2006.

\bibitem[Jang et~al.(2016)Jang, Gu, and Poole]{jang2016categorical}
Eric Jang, Shixiang Gu, and Ben Poole.
\newblock Categorical reparameterization with gumbel-softmax.
\newblock \emph{arXiv preprint arXiv:1611.01144}, 2016.

\bibitem[Jordan et~al.(1999)Jordan, Ghahramani, Jaakkola, and
  Saul]{jordan1999introduction}
Michael~I Jordan, Zoubin Ghahramani, Tommi~S Jaakkola, and Lawrence~K Saul.
\newblock An introduction to variational methods for graphical models.
\newblock \emph{Machine learning}, 37\penalty0 (2):\penalty0 183--233, 1999.

\bibitem[Kim \& Bengio(2016)Kim and Bengio]{kim2016deep}
Taesup Kim and Yoshua Bengio.
\newblock Deep directed generative models with energy-based probability
  estimation.
\newblock \emph{arXiv preprint arXiv:1606.03439}, 2016.

\bibitem[Kingma \& Ba(2014)Kingma and Ba]{kingma2014adam}
Diederik~P Kingma and Jimmy Ba.
\newblock Adam: A method for stochastic optimization.
\newblock \emph{arXiv preprint arXiv:1412.6980}, 2014.

\bibitem[Kingma \& Welling(2013)Kingma and Welling]{kingma2013auto}
Diederik~P Kingma and Max Welling.
\newblock Auto-encoding variational bayes.
\newblock \emph{arXiv preprint arXiv:1312.6114}, 2013.

\bibitem[Kr{\"a}henb{\"u}hl \& Koltun(2011)Kr{\"a}henb{\"u}hl and
  Koltun]{krahenbuhl2011efficient}
Philipp Kr{\"a}henb{\"u}hl and Vladlen Koltun.
\newblock Efficient inference in fully connected crfs with gaussian edge
  potentials.
\newblock In \emph{Advances in neural information processing systems}, pp.\
  109--117, 2011.

\bibitem[Kuleshov \& Ermon(2017)Kuleshov and Ermon]{kuleshov2017neural}
Volodymyr Kuleshov and Stefano Ermon.
\newblock Neural variational inference and learning in undirected graphical
  models.
\newblock In \emph{Advances in Neural Information Processing Systems}, pp.\
  6734--6743, 2017.

\bibitem[Lafferty et~al.(2001)Lafferty, McCallum, and
  Pereira]{lafferty2001conditional}
John Lafferty, Andrew McCallum, and Fernando~CN Pereira.
\newblock Conditional random fields: Probabilistic models for segmenting and
  labeling sequence data.
\newblock 2001.

\bibitem[Larochelle \& Murray(2011)Larochelle and Murray]{larochelle2011}
Hugo Larochelle and Iain Murray.
\newblock The neural autoregressive distribution estimator.
\newblock In \emph{The Proceedings of the 14th International Conference on
  Artificial Intelligence and Statistics}, volume~15 of \emph{JMLR: W\&CP},
  pp.\  29--37, 2011.

\bibitem[Liu \& Wang(2017)Liu and Wang]{liu2017learning}
Qiang Liu and Dilin Wang.
\newblock Learning deep energy models: Contrastive divergence vs. amortized
  mle.
\newblock \emph{arXiv preprint arXiv:1707.00797}, 2017.

\bibitem[Maddison et~al.(2016)Maddison, Mnih, and Teh]{maddison2016concrete}
Chris~J Maddison, Andriy Mnih, and Yee~Whye Teh.
\newblock The concrete distribution: A continuous relaxation of discrete random
  variables.
\newblock \emph{arXiv preprint arXiv:1611.00712}, 2016.

\bibitem[Meng \& Wong(1996)Meng and Wong]{meng1996simulating}
Xiao-Li Meng and Wing~Hung Wong.
\newblock Simulating ratios of normalizing constants via a simple identity: a
  theoretical exploration.
\newblock \emph{Statistica Sinica}, pp.\  831--860, 1996.

\bibitem[Mnih \& Gregor(2014)Mnih and Gregor]{mnih2014neural}
Andriy Mnih and Karol Gregor.
\newblock Neural variational inference and learning in belief networks.
\newblock \emph{arXiv preprint arXiv:1402.0030}, 2014.

\bibitem[Neal(1993)]{neal1993probabilistic}
Radford~M Neal.
\newblock Probabilistic inference using markov chain monte carlo methods.
\newblock 1993.

\bibitem[Neal(2001)]{neal2001annealed}
Radford~M Neal.
\newblock Annealed importance sampling.
\newblock \emph{Statistics and computing}, 11\penalty0 (2):\penalty0 125--139,
  2001.

\bibitem[Ngiam et~al.(2011)Ngiam, Chen, Koh, and Ng]{ngiam2011learning}
Jiquan Ngiam, Zhenghao Chen, Pang~W Koh, and Andrew~Y Ng.
\newblock Learning deep energy models.
\newblock In \emph{Proceedings of the 28th international conference on machine
  learning (ICML-11)}, pp.\  1105--1112, 2011.

\bibitem[Ranganath et~al.(2014)Ranganath, Gerrish, and
  Blei]{ranganath2014black}
Rajesh Ranganath, Sean Gerrish, and David Blei.
\newblock Black box variational inference.
\newblock In \emph{Artificial Intelligence and Statistics}, pp.\  814--822,
  2014.

\bibitem[Reddi et~al.(2016)Reddi, Hefny, Sra, Poczos, and
  Smola]{reddi2016stochastic}
Sashank~J Reddi, Ahmed Hefny, Suvrit Sra, Barnabas Poczos, and Alex Smola.
\newblock Stochastic variance reduction for nonconvex optimization.
\newblock In \emph{International conference on machine learning}, pp.\
  314--323, 2016.

\bibitem[Rezende et~al.(2014)Rezende, Mohamed, and
  Wierstra]{rezende2014stochastic}
Danilo~Jimenez Rezende, Shakir Mohamed, and Daan Wierstra.
\newblock Stochastic backpropagation and approximate inference in deep
  generative models.
\newblock \emph{arXiv preprint arXiv:1401.4082}, 2014.

\bibitem[Rother et~al.(2007)Rother, Kolmogorov, Lempitsky, and
  Szummer]{rother2007optimizing}
Carsten Rother, Vladimir Kolmogorov, Victor Lempitsky, and Martin Szummer.
\newblock Optimizing binary mrfs via extended roof duality.
\newblock In \emph{2007 IEEE Conference on Computer Vision and Pattern
  Recognition}, pp.\  1--8. IEEE, 2007.

\bibitem[Salakhutdinov \& Hinton(2009)Salakhutdinov and
  Hinton]{salakhutdinov2009deep}
Ruslan Salakhutdinov and Geoffrey Hinton.
\newblock Deep {B}oltzmann machines.
\newblock In \emph{Proceedings of the twelfth international conference on
  artificial intelligence and statistics}, 2009.

\bibitem[Salakhutdinov \& Larochelle(2010)Salakhutdinov and
  Larochelle]{salakhutdinov2010efficient}
Ruslan Salakhutdinov and Hugo Larochelle.
\newblock Efficient learning of deep {B}oltzmann machines.
\newblock In \emph{Proceedings of the thirteenth international conference on
  artificial intelligence and statistics}, pp.\  693--700, 2010.

\bibitem[Salakhutdinov \& Murray(2008)Salakhutdinov and
  Murray]{salakhutdinov2008quantitative}
Ruslan Salakhutdinov and Iain Murray.
\newblock On the quantitative analysis of deep belief networks.
\newblock In \emph{Proceedings of the 25th international conference on Machine
  learning}, pp.\  872--879. ACM, 2008.

\bibitem[Schulman et~al.(2015)Schulman, Heess, Weber, and
  Abbeel]{schulman2015gradient}
John Schulman, Nicolas Heess, Theophane Weber, and Pieter Abbeel.
\newblock Gradient estimation using stochastic computation graphs.
\newblock In \emph{Advances in Neural Information Processing Systems}, pp.\
  3528--3536, 2015.

\bibitem[Tieleman(2008)]{tieleman2008training}
Tijmen Tieleman.
\newblock Training restricted boltzmann machines using approximations to the
  likelihood gradient.
\newblock In \emph{Proceedings of the 25th international conference on Machine
  learning}, pp.\  1064--1071. ACM, 2008.

\bibitem[Wainwright \& Jordan(2006)Wainwright and Jordan]{wainwright2006log}
Martin~J Wainwright and Michael~I Jordan.
\newblock Log-determinant relaxation for approximate inference in discrete
  markov random fields.
\newblock \emph{IEEE transactions on signal processing}, 54\penalty0
  (6):\penalty0 2099--2109, 2006.

\bibitem[Wainwright et~al.(2005)Wainwright, Jaakkola, and
  Willsky]{wainwright2005new}
Martin~J Wainwright, Tommi~S Jaakkola, and Alan~S Willsky.
\newblock A new class of upper bounds on the log partition function.
\newblock \emph{IEEE Transactions on Information Theory}, 51\penalty0
  (7):\penalty0 2313--2335, 2005.

\bibitem[Wang et~al.(2017)Wang, Ma, Goldfarb, and Liu]{wang2017stochastic}
Xiao Wang, Shiqian Ma, Donald Goldfarb, and Wei Liu.
\newblock Stochastic quasi-newton methods for nonconvex stochastic
  optimization.
\newblock \emph{SIAM Journal on Optimization}, 27\penalty0 (2):\penalty0
  927--956, 2017.

\bibitem[Welling \& Hinton(2002)Welling and Hinton]{welling2002new}
Max Welling and Geoffrey~E Hinton.
\newblock A new learning algorithm for mean field boltzmann machines.
\newblock In \emph{International Conference on Artificial Neural Networks},
  pp.\  351--357. Springer, 2002.

\bibitem[Welling \& Sutton(2005)Welling and Sutton]{welling2005learning}
Max Welling and Charles~A Sutton.
\newblock Learning in markov random fields with contrastive free energies.
\newblock In \emph{AISTATS}. Citeseer, 2005.

\bibitem[Winn \& Bishop(2005)Winn and Bishop]{winn2005variational}
John Winn and Christopher~M Bishop.
\newblock Variational message passing.
\newblock \emph{Journal of Machine Learning Research}, 6\penalty0
  (Apr):\penalty0 661--694, 2005.

\bibitem[Zhai et~al.(2016)Zhai, Cheng, Feris, and Zhang]{zhai2016generative}
Shuangfei Zhai, Yu~Cheng, Rogerio Feris, and Zhongfei Zhang.
\newblock Generative adversarial networks as variational training of energy
  based models.
\newblock \emph{arXiv preprint arXiv:1611.01799}, 2016.

\end{thebibliography}

\appendix

\section{Derivation of the objective function}
\label{app:obj}

Here we derive the objective function of AdVIL in detail. Let $\theta$, $\phi$ and $\psi$ denote the trainable parameters in the MRF, the variational encoder and the variational decoder, respectively. The first variational trick bounds the free energy as follows:
\begin{align*}
\mathcal{L}(\theta) & = -\mathbb{E}_{P_\mathcal{D}(v)}\left[\log(\int_h e^{-  \mathcal{E}(v, h)} dh)\right] + \log \mathcal{Z} \\
& = -\mathbb{E}_{P_\mathcal{D}(v)}\log[\int_h Q(h|v)\frac{e^{-  \mathcal{E}(v, h)}}{Q(h|v)}dh]  + \log \mathcal{Z} \\
& \le \mathbb{E}_
{P_\mathcal{D}(v)Q(h|v)} \left[\mathcal{E} (v, h) + \log Q(h | v) \right] + \log \mathcal{Z} \coloneqq \mathcal{L}_1(\theta, \phi),
\end{align*}
where the bound is derived via applying the Jensen inequality and the equality holds if and only if $Q(h|v) = P(h|v)$ for all $v$.

The second variational trick bounds the log partition function as follows:
\begin{align*}
\mathcal{L}_1(\theta, \phi) \nonumber & =  \mathbb{E}_{P_\mathcal{D}(v)Q(h|v)} \left[\mathcal{E} (v, h) + \log Q(h | v) \right] + \log (\int_v\int_h e^{-\mathcal{E}(v, h)} dv dh)  \\
& =  \mathbb{E}_{P_\mathcal{D}(v)Q(h|v)} \left[\mathcal{E} (v, h) + \log Q(h | v) \right] + \log (\int_v\int_h q(v, h) \frac{e^{-\mathcal{E}(v, h)}}{q(v, h)} dv dh) \\
& \ge  \mathbb{E}_{P_\mathcal{D}(v)Q(h|v)} \left[\mathcal{E} (v, h) + \log Q(h | v) \right] + \mathbb{E}_{q(v, h)} \left [\log (\frac{ e^{-\mathcal{E}(v, h)}}{q(v, h)}) \right] \\
& =  \mathbb{E}_{P_\mathcal{D}(v)Q(h|v)}\left[\underbrace{\overbrace{\mathcal{E}(v, h)}^{\text{energy term}} +  \overbrace{\log Q(h|v)}^{\text{entropy term}}}_{\textbf{Positive Phase}}\right]
 - \mathbb{E}_{q(v, h)}\left[\underbrace{\overbrace{\mathcal{E}(v, h)}^{\text{energy term}} +  \overbrace{\log q(v, h)}^{\text{entropy term}}}_{\textbf{Negative Phase}}\right] \\
& \coloneqq \mathcal{L}_2(\theta,\phi,\psi),
\end{align*}
where the bound is also derived via applying the Jensen inequality and the equality holds if and only if $q(v, h) = P(v, h)$.

To enhance the expressive power of the variational decoder, we introduce an auxiliary variable $z$ and define $q(v, h) = \int_z q(z) q(h | z) q(v | h) dz$, which makes the entropy term in the negative phase intractable. To address the problem, we propose the third variational approximation. First, we can decompose the entropy of $q(v, h)$ as $- \mathbb{E}_{q(v, h)} \log q(v, h) =  - \mathbb{E}_{q(v, h)}  \log q(v| h) - \mathbb{E}_{q(h)} \log q( h)$ and we only need to approximate  $- \mathbb{E}_{q(h)} \log q( h)$. However, simply applying the standard variational trick as above, we get an upper bound as follows:
\begin{align*}
- \mathbb{E}_{q(h)} \log q( h)  & =   - \mathbb{E}_{q(h)} \log \int_z q( h, z)  dz  \\
& =   - \mathbb{E}_{q(h)} \log \int_z r(z|h) \frac{q( h, z)}{r(z|h)}  dz  \\
& \le  - \mathbb{E}_{  q(h) r(z|h)}  \log \left [\frac{q(h, z)}{r(z|h)} \right ],
\end{align*}
which is not satisfactory because the optimization problem will be $\min_P \min_Q \max_q \min_r$. Instead, we derive a lower bound as follows:


\begin{align*}
- \mathbb{E}_{q(h)}  \log q(h)
& =   - \mathbb{E}_{q(h)}  \log q(h) - \mathbb{E}_{q(h, z)}  \log q(z|h) + \mathbb{E}_{q(h, z)}  \log q(z|h) \\
& =  - \mathbb{E}_{q(h, z)}  \log q(h, z) + \mathbb{E}_{q(h, z)}  \log q(z|h)  \\
& =   - \mathbb{E}_{  q(h, z)}  \log \left [\frac{q(h, z)}{r(z|h)} \right ]   + \mathbb{D}_{KL}( q(z|h) || r(z|h))  \\
& \ge  - \mathbb{E}_{  q(h, z)}  \log \left [\frac{q(h, z)}{r(z|h)} \right ], \\
\end{align*}
where $\mathbb{D}_{KL}(\cdot || \cdot)$ denotes the KL-divergence and the equality holds if and only if $r(z|h) = q(z|h)$ for all $h$. The difference between the two bounds is that the expectation 
is taken over $q(h) r(z|h)$ in the upper bound while over $q(h, z)$ in the lower bound. Using the lower bound, the optimization problem will be $\min_P \min_Q \max_p \max_r$.

\section{Formal training procedure}
\label{app:algo}

\begin{algorithm}[t]
\begin{algorithmic}[1]
\caption{Adversarial variational inference and learning by stochastic gradient descent}\label{algo:AdVIL}
\STATE {\bfseries Input:} Constants $K_1$ and $K_2$, learning rate schemes $\alpha$ and $\gamma$, randomly initialized $\theta$, $\phi$ and $\psi$
\REPEAT
 \FOR{$i = 1, ..., K_1$}
\STATE Sample a batch of $(v, h, z) \sim q(v, h, z)$
\STATE Estimate the objective of $q$ and $r$ according to Eqn.~(\ref{eqn:entropy_bound}) and the negative phase in Eqn.~(\ref{eqn:AdVIL_obj})
\STATE Update $\psi$ to maximize the objective according to $\alpha$
\ENDFOR
 \FOR{$i = 1, ..., K_2$}
\STATE Sample a batch of $(v, h) \sim P_{\mathcal{D}} (v) Q(h|v)$
\STATE Estimate the objective of $Q$ according to the positive phase in Eqn.~(\ref{eqn:AdVIL_obj})
\STATE Update $\phi$ to minimize the objective according to $\gamma$
\ENDFOR
\STATE Sample a batch of $(v, h) \sim P_{\mathcal{D}} (v) Q(h|v)$ and another batch of $(v, h) \sim q(v, h)$
\STATE Estimate the objective of $P$ according to Eqn.~(\ref{eqn:AdVIL_obj})
\STATE Update $\theta$ to minimize the objective according to $\gamma$
\UNTIL{Convergence or reaching certain threshold}
\end{algorithmic}
\end{algorithm}

The formal training procedure of AdVIL is presented in Algorithm~\ref{algo:AdVIL}.

\section{Detailed theoretical analysis}
\label{app:theory}

For simplicity, we consider discrete $v$ and $h$ (e.g., in an RBM) and the analysis can be extended to the continuous cases. We assume  $v \in \{0, 1\}^{d_v}$ and $h \in \{0, 1\}^{d_h}$, where $d_v$ and $d_h$ are the dimensions of the visible and latent variables respectively.

\subsection{Analysis in the nonparametric case}
\label{app:theorey_nop}

We first analyze the nonparametric case in Proposition~\ref{thm:nop} as follows.

\begin{proposition}~\label{thm:nop}
For any $P(v, h) = \exp(- \mathcal{E}(v, h)) / \mathcal{Z}$, $\mathcal{L}_2(\theta, \phi, \psi)$ is a tight estimate of the negative log-likelihood of $P(v) $, under the following assumptions
\begin{enumerate}
    \item $Q(h|v)$ and $q(v, h)$ are nonparametric.
    \item The inner optimization over $Q(h|v)$ and $q(v, h)$ can get their optima.
\end{enumerate}
\end{proposition}

\begin{proof}
Given $P(v, h)$, i.e., $\mathcal{E} (v, h)$, to find $q^*(v, h)$, we optimize $\mathcal{L}_2$ over $\{ q(v, h)| v \in \{0, 1\}^{d_v}, h \in \{0, 1\}^{d_h}\}$ (we will use a shortcut $\{ q(v, h)\}$ for simplicity). The optimization problem is equivalent to: 
\begin{align*} 
\min_{ \{q(v, h)\} }
 & \sum_{v, h} q(v, h) \left[ \mathcal{E}(v, h)  + \log q(v, h) \right]  \\
 \textrm{subject to:} & \sum_{v, h} q(v, h) = 1, \\
  & q(v, h) \ge 0, \forall v, h.
\end{align*}
Note that the objective function is convex since its Hessian matrix is positive semi-definite. Besides, the constraints are linear. Therefore, it is a convex optimization problem. Further, we can verify that the Slater's condition~\citep{boyd2004convex} holds when $q$ is uniform and then the strong duality holds. Then, we can use the KKT conditions to solve the optimization problem.

The Lagrangian $\mathcal{G} (\{q(v, h)\} , \lambda, \{\mu(v, h)\})$ is:
\begin{align*} 
\sum_{v, h} q(v, h) \left[ \mathcal{E}(v, h)  + \log q(v, h) \right] + \lambda (\sum_{v, h} q(v, h) - 1) + \sum_{v, h} \mu(v, h) q(v, h),
\end{align*}
where  $\lambda$ and $\{\mu(v, h)\}$ are the associated Lagrange multipliers.

To satisfy the stationarity, we take gradients with respect to $q(v, h)$ for all $(v, h)$ and get: 
\begin{align*} 
 \left[ \mathcal{E}(v, h)  + \log q^*(v, h)  + 1 \right] + \lambda +  \mu(v, h) = 0,
\end{align*}
which implies 
\begin{align*} 
 q^*(v, h)   =  \exp(- \mathcal{E}(v, h) - (1 + \lambda + \mu(v, h))).
\end{align*}
According to the complementary slackness, we have 
\begin{align*} 
 \mu(v, h) q^*(v, h) = 0, \forall v,h,
\end{align*}
which implies $\mu(v, h) = 0, \forall v, h$, since  $ q^*(v, h) > 0 , \forall v, h$.

To satisfy the primal equality constraint, we have 
\begin{align*} 
\sum_{v, h} q^*(v, h)  =  \sum_{v, h} \exp(- \mathcal{E}(v, h) - (1 + \lambda)) = 1,
\end{align*}
which implies
\begin{align*} 
 q^*(v, h)  =  \frac{ \exp(- \mathcal{E}(v, h))}{\sum_{v', h'} \exp(- \mathcal{E}(v', h'))} = P(v, h), \forall v, h. 
\end{align*}

To find $Q^*(h|v)$, we optimize $\mathcal{L}_2$ over $\{ Q(h|v)| v \in \{0, 1\}^{d_v}, h \in \{0, 1\}^{d_h}\}$ (we will use a shortcut $\{ Q(h|v)\}$ for simplicity). The optimization problem  is equivalent to: 
\begin{align*} 
\min_{ \{Q(h|v)\} }
 & \sum_{v} P_{\mathcal{D}}(v)\sum_{h} Q(h|v) \left[ \mathcal{E}(v, h)  + \log Q(h|v) \right]  \\
 \textrm{subject to:} & \sum_{h} Q(h|v) = 1, \forall v, \\
  & Q(h|v) \ge 0, \forall v, h.
\end{align*}
Similar to the above procedure, we can get 
\begin{align*} 
 Q^*(h|v)  =  \frac{ \exp(- \mathcal{E}(v, h))}{\sum_{ h'} \exp(- \mathcal{E}(v, h'))} = P(h| v), \forall v, h. 
\end{align*}

Under the assumptions that (1) $Q(h|v)$ and $q(v, h)$ are nonparametric, and (2) the inner optimization over $\psi$ and $\phi$ can get the optimum, the optimal variational distributions $P(v, h)$ and  $P(h | v)$ can be obtained. Plugging them back into $\mathcal{L}_2$, we get
\begin{align*} 
\mathcal{L}_2 & = \mathbb{E}_{P_\mathcal{D}(v)P(h|v)} \left[ \mathcal{E}(v, h)  + \log P( h | v) \right] - \mathbb{E}_{P(v, h)} \left[ \mathcal{E}(v, h)  + \log P(v, h) \right] \\
 & = \mathbb{E}_{P_\mathcal{D}(v)P(h|v)} \left[ - \log \sum_h e^{-\mathcal{E} (v, h)} \right] + \mathbb{E}_{P(v, h)} \left[ \log \mathcal{Z} \right]\\
 & = \mathbb{E}_{P_\mathcal{D}(v)} \left[ \mathcal{F} (v)\right] + \log \mathcal{Z}  = \mathcal{L}.
\end{align*}
\end{proof}

\textbf{Remark~} Similar to Theorem 1 in~\citep{goodfellow2014generative}, Proposition~\ref{thm:nop} is under the nonparametric assumption, which is relaxed in our following analysis. Namely, we will consider more practical cases where $q(v, h)$ may not be exactly the same as $P(v, h)$ during training. 

\subsection{Main convergence theorem}
\label{app:theorey_convergence}

For convenience, we summarize the training dynamics of Algorithm~\ref{algo:AdVIL} with $K_1 = 1, K_2 = 1$ and the exact gradients (not the stochastic ones), as follows: 
\begin{align} 
\psi_{k+1} = \psi_{k} + \alpha_k \frac{\partial \mathcal{L}_2(\theta_{k}, \phi_{k}, \psi_{k})}{\partial \psi}, \nonumber
\\
\phi_{k+1} = \phi_{k} - \gamma_k \frac{\partial \mathcal{L}_2(\theta_{k}, \phi_{k}, \psi_{k+1})}{\partial \phi}, \nonumber \\ 
\theta_{k+1} = \theta_{k} - \gamma_k \frac{\partial \mathcal{L}_2(\theta_{k}, \phi_{k}, \psi_{k+1})}{\partial \theta}, \label{eqn:dynamics}
\end{align}
where $k = 1, 2, ...~$. We will prove that even though we are optimizing $\mathcal{L}_2(\theta, \phi, \psi)$, $ (\theta_k, \phi_k)$ converges to a stationary point of $\mathcal{L}_1(\theta, \phi)$  under certain conditions. To establish this, we first prove that the angle between $\frac{\partial \mathcal{L}_2(\theta, \phi, \psi)}{\partial \theta}$ and $\frac{\partial \mathcal{L}_1(\theta, \phi)}{\partial \theta}$ are sufficiently positive if $q(v, h)$ and $P(v, h)$ satisfy certain conditions, as summarized in Lemma~\ref{thm:pos_grad}.

\begin{lemma}\label{thm:pos_grad}
For any $(\theta, \phi)$, there exists a symmetric positive definite matrix $H$ such that $\frac{\partial \mathcal{L}_2(\theta, \phi, \psi)}{\partial \theta}  = H \frac{\partial \mathcal{L}_1(\theta, \phi)}{\partial \theta}$  under the  assumption: $||\sum_{v, h}\delta(v, h) \frac{\partial \mathcal{E}(v, h) }
{\partial \theta}||_2  < || \frac{\partial \mathcal{L}_1(\theta, \phi)}{\partial \theta} ||_2$ if  $ || \frac{\partial \mathcal{L}_1(\theta, \phi)}{\partial \theta} ||_2 > 0 $ and $||\sum_{v, h}\delta(v, h) \frac{\partial \mathcal{E}(v, h) }
{\partial \theta}||_2 = 0$ if 
$|| \frac{\partial \mathcal{L}_1(\theta, \phi)}{\partial \theta} ||_2 = 0$,
where $\delta(v, h) = q(v, h) - P(v, h)$.
\end{lemma}

\begin{proof}
According to the Cauchy-Schwarz inequality, we have
\begin{align*} 
\langle \frac{\partial \mathcal{L}_1(\theta, \phi)}{\partial \theta} , \sum_{v, h} \delta(v, h) \frac{\partial \mathcal{E}(v, h)}{\partial \theta} \rangle \le || \frac{\partial \mathcal{L}_1(\theta, \phi)}{\partial \theta} ||_2 || \sum_{v, h}\delta(v, h) \frac{\partial \mathcal{E}(v, h) }{\partial \theta}||_2 .
\end{align*}

If $|| \frac{\partial \mathcal{L}_1(\theta, \phi)}{\partial \theta} ||_2 > 0$, according to the assumption $||\sum_{v, h}\delta(v, h) \frac{\partial \mathcal{E}(v, h) }{\partial \theta}||_2  < || \frac{\partial \mathcal{L}_1(\theta, \phi)}{\partial \theta} ||_2$, we have
\begin{align*} 
\langle \frac{\partial \mathcal{L}_1(\theta, \phi)}{\partial \theta} , \sum_{v, h} \delta(v, h) \frac{\partial \mathcal{E}(v, h) }{\partial \theta} \rangle < || \frac{\partial \mathcal{L}_1(\theta, \phi)}{\partial \theta} ||_2^2 = \langle \frac{\partial \mathcal{L}_1(\theta, \phi)}{\partial \theta} , \frac{\partial \mathcal{L}_1(\theta, \phi)}{\partial \theta}\rangle,
\end{align*}
which implies that
\begin{align*} 
\langle \frac{\partial \mathcal{L}_1(\theta, \phi)}{\partial \theta} , \frac{\partial \mathcal{L}_1(\theta, \phi)}{\partial \theta} - \sum_{v, h} \delta(v, h) \frac{\partial \mathcal{E}(v, h) }{\partial \theta} \rangle >0.
\end{align*}
According to the definitions of $\mathcal{L}_1(\theta, \phi)$ and $\mathcal{L}_2(\theta, \phi, \psi)$, 
we have
\begin{align*} 
& \frac{\partial \mathcal{L}_1(\theta, \phi)}{\partial \theta} - \sum_{v, h} \delta(v, h) \frac{\partial \mathcal{E}(v, h) }{\partial \theta} \\
= & \sum_{v, h} [P_\mathcal{D}(v)Q(h|v) - P(v, h)]  \frac{\partial \mathcal{E}(v, h) }{\partial \theta} - \sum_{v, h} [q(v, h) - P(v, h)]  \frac{\partial \mathcal{E}(v, h) }{\partial \theta} \\
= & \sum_{v, h} [P_\mathcal{D}(v)Q(h|v) - q(v, h)]  \frac{\partial \mathcal{E}(v, h) }{\partial \theta} = \frac{\partial \mathcal{L}_2(\theta, \phi, \psi)}{\partial \theta},
\end{align*}
which implies that 
\begin{align*} 
\langle \frac{\partial \mathcal{L}_1(\theta, \phi)}{\partial \theta} , \frac{\partial \mathcal{L}_2(\theta, \phi, \psi)}{\partial \theta} \rangle >0.
\end{align*}
Equivalently, there exists a symmetric positive definite matrix $H$ such that $\frac{\partial \mathcal{L}_2(\theta, \phi, \psi)}{\partial \theta}  = H \frac{\partial \mathcal{L}_1(\theta, \phi)}{\partial \theta}.$
Note that this also holds when $|| \frac{\partial \mathcal{L}_1(\theta, \phi)}{\partial \theta} ||_2 = 0$ (i.e., $ \frac{\partial \mathcal{L}_1(\theta, \phi)}{\partial \theta} = \vec{0}$) because  $|| \frac{\partial \mathcal{L}_2(\theta, \phi, \psi)}{\partial \theta} ||_2 \le || \frac{\partial \mathcal{L}_1(\theta, \phi)}{\partial \theta} ||_2 + ||\sum_{v, h}\delta(v, h) \frac{\partial \mathcal{E}(v, h) }{\partial \theta}||_2 = 0 $ (i.e., $ \frac{\partial \mathcal{L}_2(\theta, \phi, \psi)}{\partial \theta} = \Vec{0}$), according to the assumption.
\end{proof}

\textbf{Remark~} 
Lemma~\ref{thm:pos_grad} assumes that $q(v, h)$ and $P(v, h)$ are sufficiently close, which is encouraged by choosing a sufficiently powerful family of $q(v, h)$ and updating $\psi$ multiple times per update of $\theta$, i.e. $K_1 > 1$. If Lemma~\ref{thm:pos_grad} holds, optimizing $\mathcal{L}_2(\theta, \phi, \psi)$ with respect to $\theta$ will decrease $\mathcal{L}_1(\theta, \phi)$ in expectation with a sufficiently small stepsize. Also note that for any $(\theta, \psi)$, $\frac{\partial \mathcal{L}_2(\theta, \phi, \psi)}{\partial \phi} = \frac{\partial \mathcal{L}_1(\theta, \phi)}{\partial \phi}$ and therefore, optimizing $\mathcal{L}_2(\theta, \phi, \psi)$ with respect to $\phi$ will decrease $\mathcal{L}_1(\theta, \phi)$ in expectation with a sufficiently small stepsize.

Based on Lemma~\ref{thm:pos_grad} and other commonly used assumptions in the analysis of stochastic gradient descent~\citep{bottou2018optimization}, Algorithm~\ref{algo:AdVIL} converges to a stationary point of $\mathcal{L}_1(\theta, \phi)$, as stated in Theorem~\ref{thm:convergence}.

\begin{theorem}\label{thm:convergence}
Solving the optimization problem in Eqn.~(\ref{eqn:lp}) using stochastic gradient descent according to Algorithm~\ref{algo:AdVIL}, then
\begin{align*} 
\lim_{k \rightarrow \infty} \ep [||\frac{\partial \mathcal{L}_1(\theta_k, \phi_k)}{\partial\theta}||_2^2] = 0, 
\end{align*}
under the following assumptions.
\begin{enumerate}
    \item The condition of Corollary 4.12 in~\citep{bottou2018optimization}: $\mathcal{L}_2(\theta, \phi, \psi)$ is twice differentiable with respect to $\theta$, $\phi$ and $\psi$.
    \item Assumption 4.1 in~\citep{bottou2018optimization}: the gradients of $\mathcal{L}_2(\theta, \phi, \psi)$ with respect to $\theta$, $\phi$ and $\psi$ are Lipschitz.
    \item Assumption 4.3 in~\citep{bottou2018optimization}: the first and second moments of the stochastic gradients are bounded by the expected gradients. 
    \item The stepsize satisfies the diminishing condition~\citep{bottou2018optimization}, i.e., 
    $\alpha_k = \gamma_k$, $\sum_{k=1}^\infty \gamma_k = \infty$, $\sum_{k=1}^\infty \gamma_k^2 < \infty$.
    \item The condition of Lemma~\ref{thm:pos_grad} holds in each step $k$. Therefore, \begin{align*}
\forall k, \exists H_k, \frac{\partial \mathcal{L}_2(\theta_k, \phi_k, \psi_{k+1})}{\partial \theta} = H_k \frac{\partial \mathcal{L}_1(\theta_k, \phi_k)}{\partial \theta}.
\end{align*}
\end{enumerate}
\end{theorem}

\begin{proof}
See Corollary 4.12 in~\citep{bottou2018optimization}.
\end{proof}

\textbf{Remark~} Assumption 1 and Assumption 2 in Theorem~\ref{thm:convergence} are ensured because we use the sigmoid and tanh activation functions. Assumption 3 and Assumption 4 in Theorem~\ref{thm:convergence} are ensured by the sampling and learning rate schemes of the stochastic gradient descent. Assumption 5 in Theorem~\ref{thm:convergence} is weaker than the nonparametric assumption of Proposition~\ref{thm:nop} but still requires a large $K_1$. Also note that the statement of converging to $\mathcal{L}_1$ in Theorem~\ref{thm:convergence} is weaker than that in Proposition~\ref{thm:nop}. 

\subsection{Complementary convergence theorem}
\label{app:theorey_complementary}

\citet{heusel2017gans} propose a two-time scale update rule to train minimax optimization problems with a convergence guarantee even using $K_1 = 1$. AdVIL converges if using the same training method as in~\citep{heusel2017gans}, which is summarized in Proposition~\ref{thm:another_conv}.
\begin{proposition}\label{thm:another_conv}
AdVIL trained with a two-time scale update rule~\citep{heusel2017gans} converges to a stationary local Nash equilibrium almost surely under the following assumptions.
\begin{enumerate}
\item The gradients with respect to $\theta$, $\phi$ and $\psi$ are Lipschitz.
\item $\sum_k \alpha_k = \infty$, $\sum_k \alpha_k^2 < \infty$, $\sum_k \gamma_k = \infty$, $\sum_k \gamma_k^2 < \infty$,$\gamma_k = o(\alpha_k)$.
\item The stochastic gradient errors are bounded in expectation.
\item For each $\theta$, the ordinary differentiable equation corresponding to Equation~\ref{eqn:dynamics} has a local asymptotically stable attractor within a domain of attraction such that the attractor is Lipschitz. Similar assumptions are required for $\phi$ and $\psi$.
\item $\sup_k ||\theta_k|| < \infty$, $\sup_k ||\psi_k|| < \infty$, $\sup_k ||\phi_k|| < \infty$.
\end{enumerate}

\end{proposition}
\begin{proof}
See Theorem 1 in~\citep{heusel2017gans}. 
\end{proof}

\textbf{Remark~} Compared to Theorem~\ref{thm:convergence}, Proposition~\ref{thm:another_conv} ensures the convergence of AdVIL without assuming $q(v, h)$ is sufficiently close to $P(v, h)$ in each step. However, a two time-sclae update rule~\citep{heusel2017gans} is required to satisfy Assumption 2 and extra weight decay terms are needed to satisfy Assumption 4. Further, the convergence point is not necessarily a stationary point of $\mathcal{L}_1(\theta, \phi)$.

\section{Datasets and experimental settings}
\label{app:data}

We evaluate our method
on the Digits dataset\footnote{ https://scikit-learn.org/stable/modules/generated/
sklearn.datasets.load$\_$digits.html$\#$sklearn.datasets.load$\_$digits}, the UCI binary databases and the Frey faces datasets\footnote{ http://www.cs.nyu.edu/$\sim$roweis/data.html}. The information of the datasets is summarized in Tab.~\ref{table:datasets}.
We implement our model using the TensorFlow~\citep{abadi2016tensorflow} library.
In all experiments, $q$ and $r$ are updated 100 times per update of $P$ and $Q$, i.e. $K_1 = 100$  and $K_2 = 1$. We use the 
ADAM~\citep{kingma2014adam} optimizer with the learning rate $\alpha  = 0.0003$, the moving average ratios $\beta_1 = 0.5$ and $\beta_2 = 0.999$, and the batch size of $500$. We use a continuous $z$ and the sigmoid activation function .
All these hyperparameters are set according to the validation performance of an RBM on the Digits dataset and fixed throughout the paper unless otherwise stated.  
The sizes of the variational distributions depend on the structure of the MRF and are chosen according to the validation performance. The model structures in RBM and DBM experiments are summarized in Tab.~\ref{table:rbm_arc} and Tab.~\ref{table:dbm_arc}, respectively. 

In a two-layer DBM with variables $v$, $h_1$ and $h_2$, we use an encoder $Q(h_1, h_2|v) = Q(h_1|v) Q(h_2|h_1)$ for both AdVIL and VCD. The decoder for AdVIL is the inverse of the encoder with one extra layer on the top, namely $q(v, h_1, h_2) = \int q(v|h_1) q(h_1|h_2) q(h_2 |z) q(z) dz$. In our implementation, both AdVIL and VCD exploits that $v$ and $h_2$ are conditionally independent given $h_1$. The layer-wise structure potentially benefits the training of both methods. Nevertheless, in principle, any differentiable variational distributions be used in AdVIL and a systematical study is left for future work.

The authors of NVIL~\citep{kuleshov2017neural} propose two variants. The first one employs a mixture of Bernoulli as $q$. The second one involves auxiliary variables and employs a neural network as $q$. Both variants scale up to an RBM of at most $64$ visible units as reported in their paper~\citep{kuleshov2017neural}. For a fair comparison, we carefully perform grid search over the default settings of NVIL and our settings based on their code and choose the best configuration. In this setting, the first variant of NVIL still fails to scale up to larger datasets and the best version of the second variant shares the same key hyperparameters as AdVIL, including $K_1 = 100$ and a batch size of 500.

\begin{table}[t]
\setlength{\tabcolsep}{4pt}
  \centering
  	\caption{Dimensions of the visible variables and sizes of the train, validation and test splits.}
\label{table:datasets}
\vskip 0.15in
  \begin{tabular}{lrrrr}
    \toprule
    	Datasets & $\#$ visible & Train & Valid. & Test \\
    \midrule
	Digits & 64 &  1438 &  359 & - \\
	Adult & 123 & 5,000 & 1414 & 26147\\
    Connect4 & 126 & 16,000 & 4000 & 47557 \\
    DNA & 180 & 1400 & 600 & 1186\\
    Mushrooms & 112 & 2,000 &  500 & 5624\\
    NIPS-0-12 & 500 & 400 & 100 & 1240\\
    OCR-letters & 128 & 32,152 & 10,000 & 10,000\\
    RCV1 & 150 & 40,000 & 10,000 & 150,000\\
    Frey faces & 560 & 1965 & - &  - \\
\bottomrule
  \end{tabular}
\end{table}

\begin{table*}[t]
\setlength{\tabcolsep}{4pt}
  \centering
  	\caption{The model structures in RBM experiments.}
\label{table:rbm_arc}
\vskip 0.15in
  \begin{tabular}{cccccccccc}
      \toprule
    	  & Digits& Adult & Connect4 & DNA & Mushrooms &  NIPS-0-12 & Ocr-letters &  RCV1  \\
    \midrule
  dimension of $z$ & $15$ & $15$ & $15$ & $15$ & $15$ & $50$ & $15$ & $15$  \\
   dimension of $h$ & $50$ & $50$ & $50$ & $50$ & $50$ & $200$ & $50$ & $50$ \\
    dimension of $v$ & $64$ & $123$ & $126$ & $180$ & $112$ & $500$ & $128$ & $150$ \\
\bottomrule
  \end{tabular}
\end{table*}

\begin{table*}[t]
\setlength{\tabcolsep}{4pt}
  \centering
  	\caption{The model structures in DBM experiments.}
\label{table:dbm_arc}
\vskip 0.15in
  \begin{tabular}{cccccccccc}
      \toprule
    	  & Digits & Adult & Connect4 & DNA & Mushrooms &  NIPS-0-12 & Ocr-letters &  RCV1  \\
    \midrule
     dimension of $z$ & $15$ & $15$ & $15$ & $15$ & $15$ & $50$ & $15$ & $15$  \\
    dimension of $h_1$ & $50$ & $50$ & $50$ & $50$ & $50$ & $200$ & $50$ & $50$  \\
    dimension of $h_2$ & $50$ & $50$ & $50$ & $50$ & $50$ & $200$ & $50$ & $50$  \\
    dimension of $v$ & $64$ & $123$ & $126$ & $180$ & $112$ & $500$ & $128$ & $150$  \\
\bottomrule
  \end{tabular}
\end{table*}

\section{More results}

\begin{figure}[t]
\begin{center}
\subfigure[]{\includegraphics[width=0.48\columnwidth]{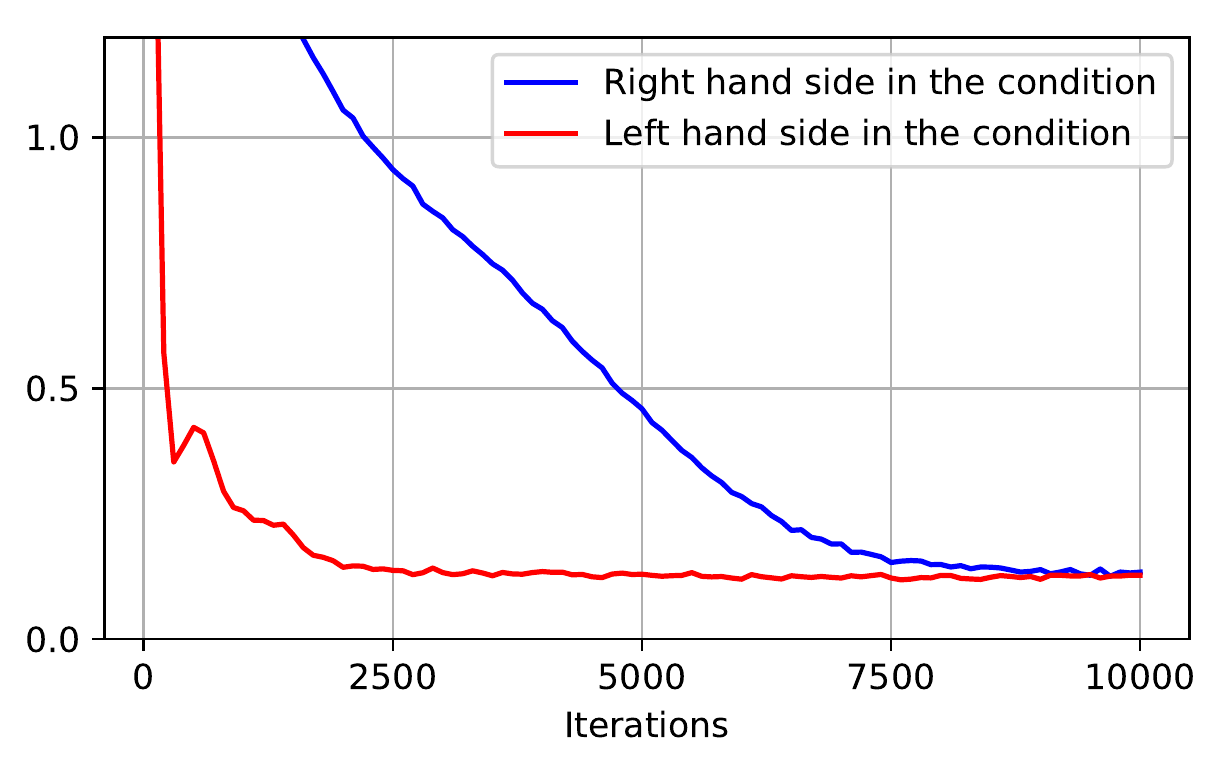}}
\subfigure[]{\includegraphics[width=0.48\columnwidth]{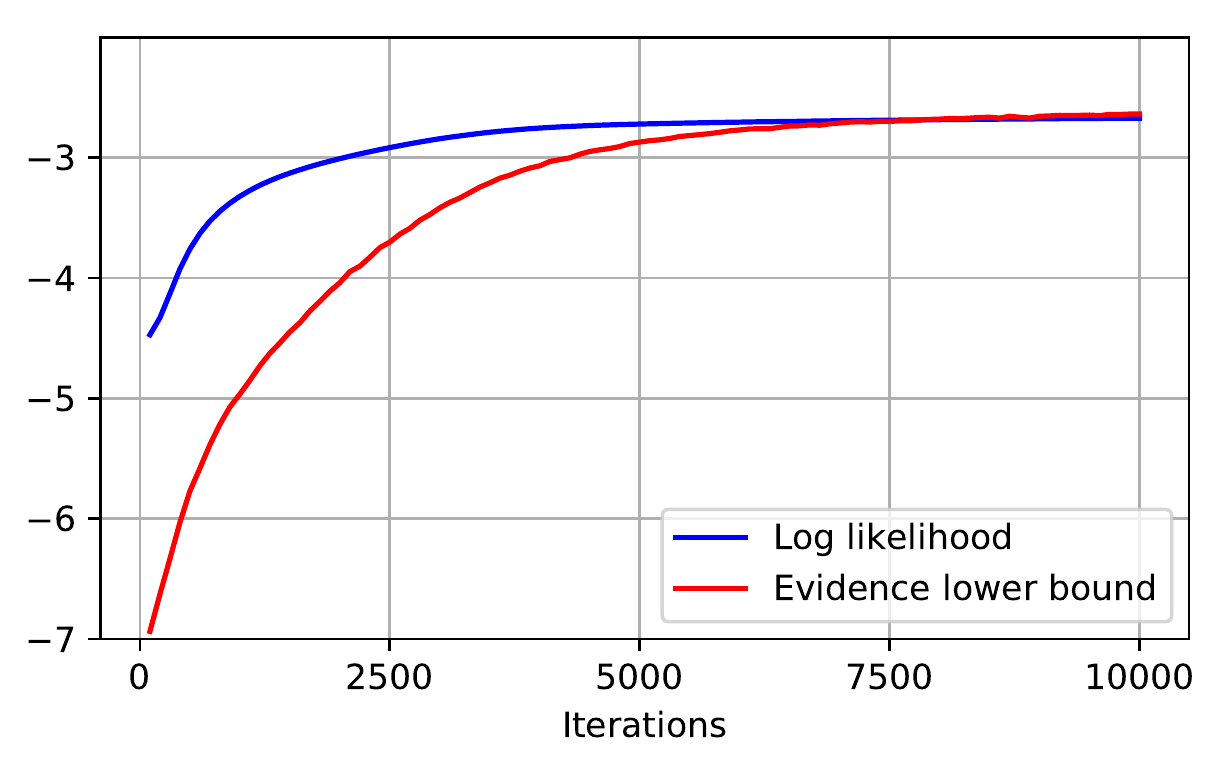}}
\end{center}
\caption{(a) shows that $||\sum_{v, h}\delta(v, h) \frac{\partial \mathcal{E}(v, h) }
{\partial \theta_k}||_2 $ (in red) is less than $|| \frac{\partial \mathcal{L}_1(\theta_k, \phi_k)}{\partial \theta_k} ||_2$ (in blue) during training. (b) shows that both the evidence lower bound (ELBO) $-\mathcal{L}_1$ (in red) and the log likelihood $-\mathcal{L}$ (in blue) converge gradually. The ELBO may be slightly over estimated because we approximate the first term in Eqn.~(\ref{eqn:AdVIL_step1}) by a Monte Carlo estimate.}
\label{fig:lemma}
\end{figure}

\subsection{Empirical verification of Theorem~\ref{thm:convergence}}
\label{app:test_lemma}

We empirically test the assumption
$||\sum_{v, h}\delta(v, h) \frac{\partial \mathcal{E}(v, h) }
{\partial \theta_k}||_2  < || \frac{\partial \mathcal{L}_1(\theta_k, \phi_k)}{\partial \theta_k} ||_2$ for $k< 10000$, where $\delta(v, h) = q(v, h) - P(v, h)$. Note that computing $ || \frac{\partial \mathcal{L}_1(\theta_k, \phi_k)}{\partial \theta_k} ||_2$ exactly requires summing over $v$ and $h$ and therefore we train a small RBM ,where the dimensions of $v$, $h$ and $z$ are $4$ on a synthetic dataset. The data distribution is a categorical distribution over $\{0, 1\}^4$ and it is sampled from a Dirichlet distribution with all concentration parameters to be one. We get 10,000, 1,000 and 1,000 i.i.d samples from the categorical distribution for training, validation and test respectively. 
We find $K_1 = 10$ is sufficient in this case.

The results are shown in Fig.~\ref{fig:lemma}. It can be seen that a decoder with neural networks and auxiliary variables are sufficiently powerful to track the model distribution and therefore Assumption 5 in Theorem~\ref{thm:convergence} holds  in 10,000 steps. Besides, the model converges gradually, which agrees with our convergence analysis. The gap between the red curve in Fig.~\ref{fig:lemma} (a) and the horizontal axis can be further reduced by using a more powerful $q(v, h)$ and advanced optimization techniques for $q(v, h)$.

Theoretically, how Lemma~\ref{thm:pos_grad} holds as the number of variables increases is not clear. Intuitively, we agree that it may get harder to satisfy this condition in a high-dimensional space. However, it is still possible with advanced optimization methods because the MRF is randomly initialized and learned gradually, and the variational decoder is trained to track the MRF (based on an old version of the decoder) after every update of the MRF. Empirically, computing both sides of the condition requires the value of the partition function, which is as hard as training the MRF. Therefore, verifying this condition during training in a high-dimensional case is highly nontrivial.

\subsection{Samples in RBM}
\label{app:rbm_samples}

\begin{figure}[t]
\begin{center}
\subfigure[$P$ normal]{\includegraphics[width=0.16\columnwidth,  trim={0 .96cm .96cm 0}, clip]{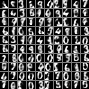}}
\subfigure[$P$ w.o. $\mathcal{H}(Q)$]{\includegraphics[width=0.16\columnwidth,  trim={0 .96cm .96cm 0}, clip]{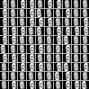}}
\subfigure[$P$ w.o. $\mathcal{H}(q)$]{\includegraphics[width=0.16\columnwidth,  trim={0 .96cm .96cm 0}, clip]{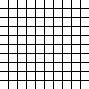}}
\subfigure[$q$ normal]{\includegraphics[width=0.16\columnwidth,  trim={0 .96cm .96cm 0}, clip]{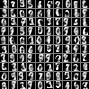}}
\subfigure[$q$ w.o. $\mathcal{H}(Q)$]{\includegraphics[width=0.16\columnwidth,  trim={0 .96cm .96cm 0}, clip]{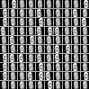}}
\subfigure[$q$ w.o. $\mathcal{H}(q)$]{\includegraphics[width=0.16\columnwidth,  trim={0 .96cm .96cm 0}, clip]{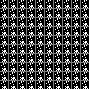}}
\vspace{-.2cm}
\caption{(a-c) Samples from the RBM in different settings. (d-f) Samples from the decoder in different settings.  We present the mean of $v$ for better visualization in all settings.}
\vspace{-.3cm}
\label{fig:entropy}
\end{center}
\end{figure}

We present the samples from the RBM $P$ and the decoder $q$ in Fig.~\ref{fig:entropy}. In this case, we set the number of the hidden units to $50$ and other settings remain the same as in Sec.~\ref{sec:analysis}.
The first column demonstrates that the decoder is a good approximate sampler for the RBM. Note that the samples from the decoder are obtained from efficient ancestral sampling but those from the RBM is obtained by Gibbs sampling after 100,000 burn-in steps.
The second column shows that if $\mathcal{H}(Q)$ is removed, both models collapse to a certain mode of the data. The third column shows that if $\mathcal{H}(q)$ is removed, both models fail to generate meaningful digits. These results demonstrate the importance of the entropy terms and the necessity of approximating $\mathcal{H}(q)$ in a principled way.

\subsection{AdVIL with an autoregressive prior}
\label{app:nade}

Here we present the results of AdVIL with a neural autoregressive distribution estimator (NADE) ~\citep{larochelle2011} as the prior on the Digits dataset. We use the same RBM as in Sec.~\ref{sec:analysis}. The dimension of the latent units in NADE is 15, which is the same as the dimension of the auxiliary variables in the hierarchical decoder presented in Sec.~\ref{sec:spe_v}.

Compared to the hierarchical decoder, the NADE decoder has a tractable entropy and hence does not require $r(z|h)$. 
However, getting samples from NADE is slow while AdVIL requires samples during training. Therefore, $K_1 = 5$ for the NADE decoder has a similar training cost as the hierarchical decoder.

\begin{figure}[t]
\begin{center}
\subfigure[AdVIL with a NADE decoder]{\includegraphics[width=0.4\columnwidth]{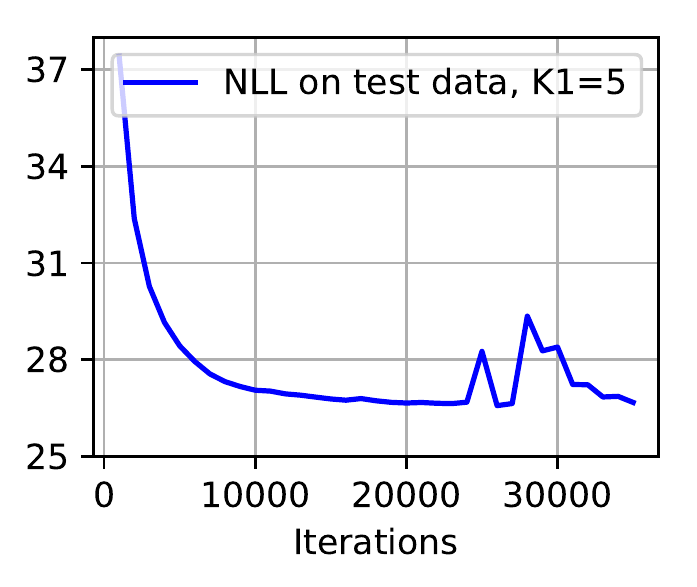}}
\subfigure[AdVIL with a hierarchial decoder]{\includegraphics[width=0.4\columnwidth]{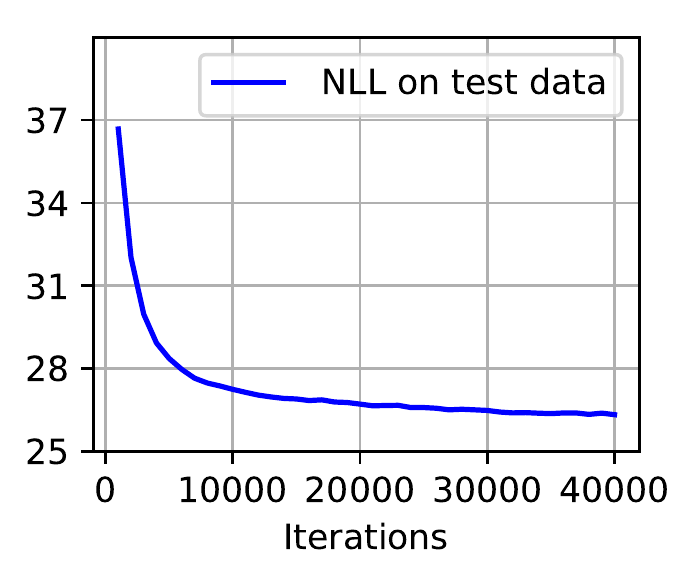}}
\caption{AdVIL with two types of decoders. We set $K_1 = 5$ in the NADE decoder and $K_1 = 100$ in the hierarchical decoder. The two models have a similar model capacity and training time.}
\label{fig:nade}
\end{center}
\end{figure}

Fig.~\ref{fig:nade} compares the two decoders. AdVIL with the NADE decoder achieves a slightly worse and unstable result.

\begin{figure}[t]
\begin{center}
\subfigure[AdVIL]{\includegraphics[width=0.4\columnwidth]{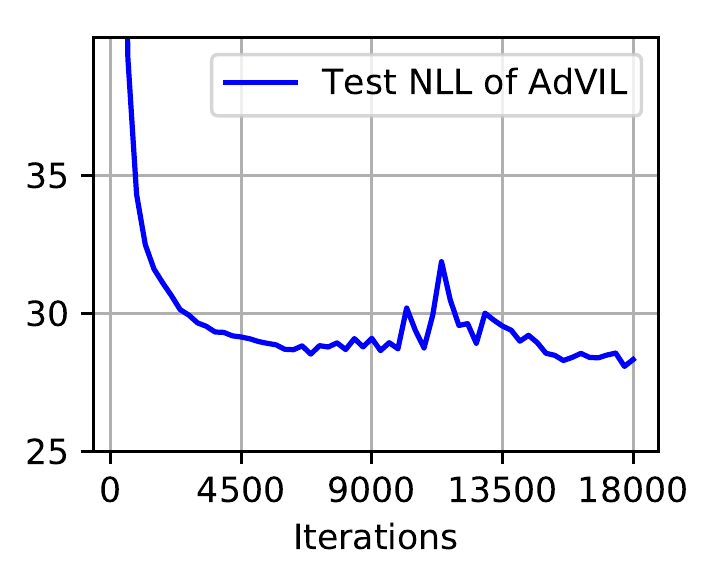}}
\subfigure[VCD]{\includegraphics[width=0.4\columnwidth]{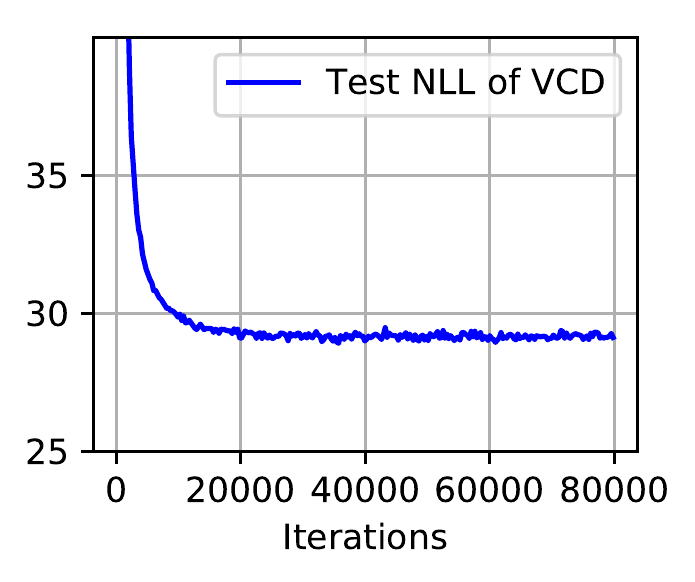}}
\caption{DBM results of AdVIL and VCD on the Digits dataset. The curve of AdVIL is less stable due to the presence of the minimax optimization problem but AdVIL  achieves a better performance.}
\label{fig:dbm_curve_digit}
\end{center}
\end{figure}

\subsection{Learning curves and analysis in DBM}
\label{app:ana_dbm}

We plot the learning curves of AdVIL and VCD in DBM, as shown in Fig.~\ref{fig:dbm_curve_digit}. AdVIL achieves a better result than VCD, which agrees with the quantitative results in Tab.~\ref{table:dbm_uci}. Note that we report the test NLL results in Tab.~\ref{table:dbm_uci} according to the best validation performance.

There are two types of biases introduced by using $Q(h_1, h_2|v) = Q(h_2|h_1) Q(h_1|v)$  in VCD. The first type of bias is introduced by using the approximate free energy in both the positive phase and the negative phase (See Eqn.~(\ref{eqn:grad_cd}) and Eqn.~(\ref{eqn:grad_vcd})). The second bias is introduced by the usage of $Q(h_1|v)$ to approximate $P(h_1|v)$ in the Gibbs sampling procedure to approximate the negative phase in Eqn.~(\ref{eqn:grad_vcd}). The influence of the two types of biases on the negative phase is not clear, which can potentially explain the relatively inferior performance of VCD in DBM. In contrast, AdVIL approximates the negative phase by introducing one bias (i.e., the approximation error between $q(v, h)$ and $P(v, h)$), whose effect on learning is theoretically characterized by Theorem 1. The above results and analysis essentially demonstrate the advantages of AdVIL over CD-based methods in DBM and the importance of developing black-box inference and learning algorithms for general MRFs.

\subsection{AIS results with standard deviations}
\label{app:std}

The AIS results in RBM and RBM with the means and standard deviations are shown in Tab.~\ref{table:rbm_more} and Tab.~\ref{table:dbm_more} respectively.

\begin{table*}
\setlength{\tabcolsep}{4pt}
  \centering
  	\caption{The AIS results of NVIL and AdVIL in RBM with the means and standard deviations. The results are averaged over three runs with different random seeds.}
\label{table:rbm_more}
\vskip 0.15in
\resizebox{\textwidth}{11.5mm}{
  \begin{tabular}{cccccccccc}
      \toprule
    	 Method & Digits & Adult & Connect4 & DNA & Mushrooms &  NIPS-0-12 & Ocr-letters &  RCV1 \\
    \midrule
NVIL-mean & $-$27.36 & $-$20.05 & $-$24.71 & $-$97.71 & $-$29.28 &$-$290.01& $-$47.56  & $-$50.47\\
NVIL-std & 0.13 & 0.27 & 0.61 & 0.12 & 0.31 &2.68 & 0.14 & 0.09\\	
AdVIL-mean& $\mathbf{-26.34}$ & $\mathbf{-19.29}$ & $\mathbf{-21.95}$ & $\mathbf{-97.59}$ & $\mathbf{-19.59}$ & $\mathbf{-276.42}$ & $\mathbf{-45.64}$ & $\mathbf{-50.22}$\\
AdVIL-std & 0.02 & 0.07 & 1.04 & 0.10 &2.01 & 0.21 & 0.34 &0.06 \\
    
\bottomrule
  \end{tabular}
 }
\end{table*}

\begin{table*}
\setlength{\tabcolsep}{4pt}
  \centering
  	\caption{The AIS results of VCD-1 and AdVIL in DBM with the means and standard deviations. The results are averaged over three runs with different random seeds.}
\label{table:dbm_more}
\vskip 0.15in
\resizebox{\textwidth}{11.5mm}{
  \begin{tabular}{cccccccccc}
      \toprule
    	 Method & Digits & Adult & Connect4 & DNA & Mushrooms &  NIPS-0-12 & Ocr-letters &  RCV1 \\
    \midrule
VCD-mean  & $-28.49$ & $-22.26$ & $-26.79$ & $\mathbf{-97.59}$ & $-23.15$ & $-356.26$  & $\mathbf{-45.77}$ & $\mathbf{-50.83}$ \\
VCD-std & 0.47 & 0.51 & 1.42 & 0.03 & 1.42 & 34.70 & 1.15 & 0.62\\
AdVIL-mean & $\mathbf{-27.89}$  & $\mathbf{-20.29}$ & $\mathbf{-26.34}$ & $-99.40$ & $\mathbf{-21.21}$ & $\mathbf{-287.15}$  & $-48.38$ & $-51.02$ \\
AdVIL-std &  0.44 & 0.24 & 1.50 & 0.71 & 0.40 & 0.63 & 1.32 &0.42 \\
    
\bottomrule
  \end{tabular}
 }
\end{table*}

\end{document}